\documentclass[11pt]{article}

\pdfoutput=1

\usepackage[numbers]{natbib}

\usepackage[T1]{fontenc}
\usepackage[latin9]{inputenc}
\usepackage{float}
\usepackage{fullpage,times,color}
\usepackage{boxedminipage}

\usepackage{amsmath,amsthm,amssymb}
\usepackage{mathtools}

\usepackage{graphicx}
\usepackage{ltxtable} 
\newcolumntype{L}{>{\centering\arraybackslash} m{0.04\columnwidth}} 
\newcolumntype{R}{>{\centering\arraybackslash} m{0.48\columnwidth}} 
\newcolumntype{S}{>{\centering\arraybackslash} m{0.32\columnwidth}} 

\usepackage{multirow}
\usepackage{array}

\usepackage{tikz}
\usepackage{pgfplots}

\newcommand{\sign}{{\mathrm {sign}}}

\newtheorem{lemma}{Lemma}

\newtheorem{theorem}{Theorem}
\newtheorem{corollary}{Corollary}

\newtheorem{remark}{Remark}
\newtheorem{openProblem}{Open Problem}

\newcommand{\E}{\mathbb{E}}

\DeclareMathOperator*{\argmin}{argmin} 
\DeclareMathOperator*{\argmax}{argmax} 
\newcommand{\reals}{\mathbb{R}}

\newcommand{\figref}[1]{Figure~\ref{#1}}

\newcommand{\secref}[1]{Section~\ref{#1}}

\newcommand{\thmref}[1]{Theorem~\ref{#1}}
\newcommand{\corref}[1]{Corollary~\ref{#1}}
\newcommand{\lemref}[1]{Lemma~\ref{#1}}

\newcommand{\half}{\frac{1}{2}}

\newenvironment{myalgo}[1]%
{
\begin{center}
\begin{boxedminipage}{0.8\linewidth}
\begin{center}
\textbf{\texttt{#1}}
\end{center}
\rm
\begin{tabbing}
....\=...\=...\=...\=...\=  \+ \kill
} %
{\end{tabbing} 
\end{boxedminipage} \end{center} 
}

\title{Accelerated Proximal Stochastic Dual Coordinate Ascent for Regularized Loss Minimization}
\author{Shai Shalev-Shwartz\thanks{School of Computer Science and
    Engineering, The Hebrew University, Jerusalem, Israel}
 \and 
Tong Zhang\thanks{Department of Statistics, Rutgers University, NJ,
  USA}~\thanks{Baidu Inc., Beijing, China} }
\date{}

\begin{document}

\maketitle

\begin{abstract}
  We introduce a proximal version of the stochastic dual coordinate
  ascent method and show how to accelerate the method using an
  inner-outer iteration procedure. We analyze the runtime of the
  framework and obtain rates that improve state-of-the-art results for
  various key machine learning optimization problems including SVM,
  logistic regression, ridge regression, Lasso, and multiclass
  SVM. Experiments validate our theoretical findings.
\end{abstract}

\section{Introduction}

We consider the following generic optimization problem associated with
regularized loss minimization of linear predictors: Let
$X_1,\ldots,X_n$ be matrices in $\reals^{d \times k}$ (referred to as instances), let
$\phi_1,\ldots,\phi_n$ be a sequence of vector convex functions
defined on $\reals^k$ (referred to as loss functions), let $g(\cdot)$ be a convex function defined on
$\reals^d$ (referred to as a regularizer), and let $\lambda \geq 0$
(referred to as a regularization parameter). Our goal is to solve:
\begin{equation} \label{eqn:PrimalProblem}
\min_{w \in \reals^d} P(w) ~~~~\textrm{where}~~~~
P(w) = \left[ \frac{1}{n} \sum_{i=1}^n \phi_i( X_i^\top w) + \lambda g(w) \right] .
\end{equation}
For example, in ridge regression the regularizer is $g(w) =
\frac{1}{2} \|w\|_2^2$, the instances are column vectors, and for
every $i$ the $i$'th loss function is $\phi_i(a) = \frac{1}{2}
(a-y_i)^2$, for some scalar $y_i$.

Let $w^* = \argmin_w P(w)$ (we will later make assumptions that imply
that $w^*$ is unique).  We say that $w$ is $\epsilon$-accurate if
$P(w) - P(w^*) \le \epsilon$. Our main result is a new algorithm for
solving \eqref{eqn:PrimalProblem}.  If $g$ is $1$-strongly convex and
each $\phi_i$ is $(1/\gamma)$-smooth (meaning that its gradient is
$(1/\gamma)$-Lipschitz), then our algorithm finds, with probability of
at least $1-\delta$, an $\epsilon$-accurate solution to
\eqref{eqn:PrimalProblem} in time
\begin{align*}
  &O\left(d\left(n+
      \min\left\{\frac{1}{\lambda\,\gamma},\sqrt{\frac{n}{\lambda\,\gamma}}\right\}\right)\log(1/\epsilon)\,
    \log(1/\delta)\,\max\{1,\log^2(1/(\lambda\,\gamma\,n))\}\right) \\
  &~~=~~ \tilde{O}\left(d\left(n+
      \min\left\{\frac{1}{\lambda\,\gamma},\sqrt{\frac{n}{\lambda\,\gamma}}\right\}\right)\right)
  ~.
\end{align*}
This applies, for example, to ridge regression and to logistic
regression with $L_2$ regularization. The $O$ notation hides
constants terms and the $\tilde{O}$ notation hides constants and
logarithmic terms. We make these explicit in the formal statement of
our theorems.

Intuitively, we can think of $\frac{1}{\lambda \gamma}$ as the
condition number of the problem. If the condition number is $O(n)$
then our runtime becomes $\tilde{O}(dn)$. This means that the
runtime is nearly linear in the data size. This matches the recent
result of \citet{ShalevZh2013,LSB12-sgdexp}, but our setting is
significantly more general. When the condition number is much larger
than $n$, our runtime becomes $\tilde{O}(d
\sqrt{\frac{n}{\lambda\,\gamma}})$. This significantly improves over
the result of \cite{ShalevZh2013,LSB12-sgdexp}. It also significantly
improves over the runtime of accelerated gradient descent due to
\citet{nesterov2007gradient}, which is
$\tilde{O}(d\,n\,\sqrt{\frac{1}{\lambda\,\gamma}})$.

By applying a smoothing technique to $\phi_i$, we also derive a method
that finds an $\epsilon$-accurate solution to
\eqref{eqn:PrimalProblem} assuming that each $\phi_i$ is
$O(1)$-Lipschitz, and obtain the runtime
\[
\tilde{O}\left(d\left(n+
  \min\left\{\frac{1}{\lambda\,\epsilon},\sqrt{\frac{n}{\lambda\,\epsilon}}\right\}\right)\right) ~.
\]
This applies, for example, to SVM with the hinge-loss. It
significantly improves over the rate $\frac{d}{\lambda \epsilon}$ of
SGD (e.g. \citep{ShalevSiSr07}), when $\frac{1}{\lambda \epsilon} \gg n$.

We can also apply our results to non-strongly convex regularizers
(such as the $L_1$ norm regularizer), or to non-regularized
problems, by adding a slight $L_2$ regularization. For example, for
$L_1$ regularized problems, and assuming that each $\phi_i$ is
$(1/\gamma)$-smooth, we obtain the runtime of
\[
\tilde{O}\left(d\left(n+
  \min\left\{\frac{1}{\epsilon\,\gamma},\sqrt{\frac{n}{\epsilon\,\gamma}}\right\}\right)\right)
~.
\]
This applies, for example, to the Lasso problem, in which the goal is
to minimize the squared loss plus an $L_1$ regularization term. 

To put our results in context, in the table below we specify the
runtime of various algorithms (while ignoring constants and
logarithmic terms) for three key machine learning applications; SVM in
which $\phi_i(a) = \max\{0,1-a\}$ and $g(w) = \frac{1}{2} \|w\|_2^2$,
Lasso in which $\phi_i(a) = \frac{1}{2}(a-y_i)^2$ and $g(w) = \sigma
\|w\|_1$, and Ridge Regression in which $\phi_i(a) =
\frac{1}{2}(a-y_i)^2$ and $g(w) = \frac{1}{2} \|w\|_2^2$. Additional
applications, and a more detailed runtime comparison to previous work,
are given in \secref{sec:applications}. In the table below, SGD stands
for Stochastic Gradient Descent, and AGD stands for Accelerated
Gradient Descent. 
\begin{center}
\setlength{\extrarowheight}{6pt}
\begin{tabular}{c|c|c} \hline
Problem & Algorithm & Runtime \\ \hline
\multirow{3}{*}{SVM} & SGD \citep{ShalevSiSr07} & $\frac{d}{\lambda \epsilon}$ \\
& AGD \citep{nesterov2005smooth} & $dn \sqrt{\frac{1}{\lambda\,\epsilon}}$ \\ 
& \textbf{This paper} &  $d\left(n +
  \min\{\frac{1}{\lambda\,\epsilon},\sqrt{\frac{n}{\lambda
      \epsilon}}\}\right)$ \\ \hline
\multirow{4}{*}{Lasso} & SGD and variants
(e.g. \citep{Zhang02-dual,Xiao10,shalev2011stochastic}) &
$\frac{d}{\epsilon^2}$ \\
& Stochastic Coordinate Descent \citep{ShalevTe09,Nesterov10} & $\frac{dn}{\epsilon}$ \\
& FISTA \citep{nesterov2007gradient,beck2009fast} & $dn \sqrt{\frac{1}{\epsilon}}$ \\ 
& \textbf{This paper} &  $d\left(n +
  \min\{\frac{1}{\epsilon},\sqrt{\frac{n}{\epsilon}}\}\right)$ \\ \hline
\multirow{4}{*}{Ridge Regression} & Exact & $d^2n + d^3$ \\
& SGD \citep{LSB12-sgdexp}, SDCA \cite{ShalevZh2013} & $d\left(n + \frac{1}{\lambda}\right)$ \\
& AGD \citep{nesterov2007gradient} & $dn \sqrt{\frac{1}{\lambda}}$ \\ 
& \textbf{This paper} &  $d\left(n +
  \min\{\frac{1}{\lambda},\sqrt{\frac{n}{\lambda}}\}\right)$ \\ \hline
\end{tabular}
\end{center}

\paragraph{Technical contribution:} Our algorithm combines two
ideas. The first is a proximal version of stochastic dual coordinate
ascent (SDCA).\footnote{Technically speaking, it may be more accurate
  to use the term \emph{randomized} dual coordinate ascent, instead of
  \emph{stochastic} dual coordinate ascent. This is because our
  algorithm makes more than one pass over the data, and therefore
  cannot work directly on distributions with infinite
  support. However, following the convention in the prior machine learning literature, we
  do not make this distinction.}  In particular, we generalize the
recent analysis of \cite{ShalevZh2013} in two directions. First, we
allow the regularizer, $g$, to be a general strongly convex function
(and not necessarily the squared Euclidean norm). This allows us to
consider non-smooth regularization function, such as the $L_1$
regularization. Second, we allow the loss functions, $\phi_i$, to be
vector valued functions which are smooth (or Lipschitz) with respect
to a general norm. This generalization is useful in multiclass
applications. As in \cite{ShalevZh2013}, the runtime of this procedure
is $\tilde{O}\left(d\left(n + \frac{1}{\lambda
      \gamma}\right)\right)$. This would be a nearly linear time (in
the size of the data) if $\frac{1}{\lambda \gamma} = O(n)$. Our second
idea deals with the case $\frac{1}{\lambda \gamma} \gg n$ by
iteratively approximating the objective function $P$ with objective
functions that have a stronger regularization. In particular, each
iteration of our acceleration procedure involves approximate
minimization of $P(w) + \frac{\kappa}{2} \|w-y\|_2^2$, with respect to
$w$, where $y$ is a vector obtained from previous iterates and
$\kappa$ is order of $1/(\gamma n)$. The idea is that the addition of
the relatively strong regularization makes the runtime of our proximal
stochastic dual coordinate ascent procedure be $\tilde{O}(dn)$. And,
with a proper choice of $y$ at each iteration, we show that the
sequence of solutions of the problems with the added regularization
converge to the minimum of $P$ after $\sqrt{\frac{1}{\lambda \gamma
    n}}$ iterations. This yields the overall runtime of $d
\sqrt{\frac{n}{\lambda \gamma}}$.

\paragraph{Additional related work:}

As mentioned before, our first contribution is a proximal version of
the stochastic dual coordinate ascent method and extension of the
analysis given in \citet{ShalevZh2013}. Stochastic dual coordinate
ascent has also been studied in \citet{CollinsGlKoCaBa08} but in more
restricted settings than the general problem considered in this
paper. One can also apply the analysis of stochastic coordinate
descent methods given in \citet{richtarik2012iteration} on the dual
problem. However, here we are interested in understanding the primal
sub-optimality, hence an analysis which only applies to the dual
problem is not sufficient.

The generality of our approach allows us to apply it for multiclass
prediction problems. We discuss this in detail later on in
\secref{sec:applications}.  Recently, \cite{lacoste2012stochastic}
derived a stochastic coordinate ascent for structural SVM based on the
Frank-Wolfe algorithm. Although with different motivations, for the
special case of multiclass problems with the hinge-loss, their
algorithm ends up to be the same as our proximal dual ascent algorithm
(with the same rate). Our approach allows to accelerate the method and
obtain an even faster rate.

The proof of our acceleration method adapts Nesterov's estimation
sequence technique, studied in \citet{OlGlNe11,ScRoBa11}, to allow
approximate and stochastic proximal mapping. See also
\cite{baes2009estimate,d2008smooth}. In particular, it relies on
similar ideas as in Proposition 4 of \cite{ScRoBa11}. However, our specific requirement
is different, and the proof presented here is different and significantly simpler than
that of \cite{ScRoBa11}.

There have been several attempts to accelerate stochastic optimization
algorithms. See for example
\cite{hu2009accelerated,ghadimi2012optimal,cotter2011better} and the
references therein. However, the runtime of these methods have a
polynomial dependence on $1/\epsilon$ even if $\phi_i$ are smooth and
$g$ is $\lambda$-strongly convex, as opposed to the logarithmic
dependence on $1/\epsilon$ obtained here. As in
\cite{LSB12-sgdexp,ShalevZh2013}, we avoid the polynomial dependence
on $1/\epsilon$ by allowing more than a single pass over the data.

\section{Preliminaries}

All the functions we consider in this paper are proper convex
functions over a Euclidean space. We use $\reals$ to denote the set of
real numbers and to simplify our notation, when we use $\reals$ to
denote the range of a function $f$ we in fact allow $f$ to output the
value $+\infty$.

Given a function $f : \reals^d \to \reals$ we denote its \textbf{conjugate}
function by \[
f^*(y) = \sup_x~ [ y^\top x - f(x)] ~.
\]
Given a norm $\|\cdot\|_P$ we denote the \textbf{dual norm} by $\|\cdot\|_D$
where
\[
\|y\|_D
= \sup_{x:\|x\|_P=1} y^\top x.
\]
We use $\|\cdot\|$ or $\|\cdot\|_2$ to denote the $L_2$ norm,
$\|x\| = x^\top x$. We also use $\|x\|_1 = \sum_i |x_i|$ and
$\|x\|_\infty = \max_i |x_i|$. The \textbf{operator norm} of a matrix
$X$ with respect to norms $\|\cdot\|_P,\|\cdot\|_{P'}$ is defined as
\[
\|X\|_{P \to P'} = \sup_{u:  \|u\|_P =1} \|X u\|_{P'} ~.
\]

A function $f: \reals^k \to \reals^d$ is $L$-\textbf{Lipschitz} with respect to a
norm $\|\cdot\|_P$, whose dual norm is $\|\cdot\|_D$, if for all
$a, b \in \reals^d$, we have
  \[
  \|f(a)- f(b)\|_D \leq L\,\|a-b\|_P .
  \]
  A function $f: \reals^d \to \reals$ is $(1/\gamma)$-\textbf{smooth} with
  respect to a norm $\|\cdot\|_P$ if it is differentiable and its
  gradient is $(1/\gamma)$-Lipschitz with respect to $\|\cdot\|_P$. An
  equivalent condition is that for all $a, b \in \reals^d$, we have
  \[
  f(a) \leq f(b) + \nabla f(b)^\top (a-b) + \frac{1}{2\gamma} \|a-b\|_{P}^2 .
  \]
A function $f : \reals^d \to \reals$ is $\gamma$-\textbf{strongly convex} with
respect to $\|\cdot\|_{P}$ if 
\[
f(w+ v) \geq f(w) + \nabla f(w)^\top v +
\frac{\gamma}{2} \|v\|_{P}^2 .
\]
It is well known that $f$ is $\gamma$-strongly convex with respect to
$\|\cdot\|_P$ if and only if $f^*$ is $(1/\gamma)$-smooth with respect
to the dual norm, $\|\cdot\|_D$. 

The \textbf{dual problem} of \eqref{eqn:PrimalProblem} is
\begin{equation} \label{eqn:DualProblem}
\max_{\alpha \in \reals^{k \times n}} D(\alpha) ~~~\textrm{where}~~~ D(\alpha) = 
\left[ \frac{1}{n} \sum_{i=1}^n -\phi_i^*(-\alpha_i) -
  \lambda g^*\left( \tfrac{1}{\lambda n} \sum_{i=1}^n X_i \alpha_i \right) \right] ~,
\end{equation}
where $\alpha_i$ is the $i$'th column of the matrix $\alpha$, which
forms a vector in $\reals^k$. 

We will assume that $g$ is strongly convex which implies that
$g^*(\cdot)$ is continuous differentiable. If we define
\begin{equation} \label{eqn:walpha}
v(\alpha)= \frac{1}{\lambda
  n} \sum_{i=1}^n X_i \alpha_i   \qquad \text{and} \qquad  w(\alpha) = \nabla g^*(v(\alpha)),
\end{equation}
then it is known that $w(\alpha^*)=w^*$, where $\alpha^*$ is an optimal solution of (\ref{eqn:DualProblem}). 
It is also known that $P(w^*)=D(\alpha^*)$ which immediately implies that for all $w$ and $\alpha$, we have
$P(w) \geq D(\alpha)$, and hence the \textbf{duality gap} defined as
\[
P(w(\alpha))-D(\alpha)
\]
can be regarded as an upper bound on both the \textbf{primal
  sub-optimality},
$P(w(\alpha))-P(w^*)$, and on the \textbf{dual sub-optimality},
$D(\alpha^*)-D(\alpha)$.

\section{Main Results}

In this section we describe our algorithms and their analysis. We
start in \secref{sec:rdca} with a description of our proximal
stochastic dual coordinate ascent procedure (Prox-SDCA). Then, in
\secref{sec:accelerate} we show how to accelerate the method by
calling Prox-SDCA on a sequence of problems with a strong
regularization. Throughout the first two sections we assume that the
loss functions are smooth. Finally, we discuss the case of Lipschitz
loss functions in \secref{sec:Lipschitz}.

The proofs of the main acceleration theorem
(\thmref{thm:acceleratedThmMain}) is given in
\secref{sec:acceleratedThmMain}. The rest of the proofs are provided
in the appendix.

\subsection{Proximal Stochastic Dual Coordinate Ascent} \label{sec:rdca}

We now describe our proximal stochastic dual coordinate ascent
procedure for solving \eqref{eqn:PrimalProblem}. Our results in this
subsection holds for $g$ being a $1$-strongly convex function with
respect to some norm $\|\cdot\|_{P'}$ and every $\phi_i$ being a
$(1/\gamma)$-smooth function with respect to some other norm
$\|\cdot\|_P$. The corresponding dual norms are denoted by
$\|\cdot\|_{D'}$ and $\|\cdot\|_D$ respectively. 

The dual objective in \eqref{eqn:DualProblem} has a different dual
vector associated with each example in the training set.  At each
iteration of dual coordinate ascent we only allow to change the $i$'th
column of $\alpha$, while the rest of the dual vectors are kept
intact. We focus on a \emph{randomized} version of dual coordinate
ascent, in which at each round we choose which dual vector to update
uniformly at random.

At step $t$, let $v^{(t-1)} = (\lambda n)^{-1} \sum_i X_i
\alpha_i^{(t-1)}$ and let $w^{(t-1)} = \nabla g^*(v^{(t-1)})$.  We
will update the $i$-th dual variable $\alpha_i^{(t)} =
\alpha_i^{(t-1)} + \Delta \alpha_i$, in a way that will lead to a
sufficient increase of the dual objective.  For the primal problem,
this would lead to the update $v^{(t)} = v^{(t-1)} + (\lambda n)^{-1}
X_i \Delta \alpha_i$, and therefore $w^{(t)} = \nabla g^*(v^{(t)})$
can also be written as
\[
w^{(t)}= \argmax_{w} \left[w^\top v^{(t)}  - g(w) \right] ~=~ 
\argmin_w \left[ - w^\top \left(n^{-1}\sum_{i=1}^n X_i
    \alpha_i^{(t)}\right) + \lambda g(w)\right] ~.
\]
Note that this particular update is rather similar to the update step
of proximal-gradient dual-averaging method (see for example
\citet{Xiao10}).  The difference is on how $\alpha^{(t)}$ is updated.

The goal of dual ascent methods is to increase the dual objective as much as possible,
and thus the optimal way to choose $\Delta \alpha_i$ would be to
maximize the dual objective, namely, we shall let
\[
\Delta \alpha_i = \argmax_{\Delta \alpha_i \in \reals^k} \left[ -\frac{1}{n} \phi^*_i(-(\alpha_i + \Delta
\alpha_i)) - \lambda g^*( v^{(t-1)} + (\lambda n)^{-1}  X_i \Delta
\alpha_i)  \right] ~.
\]
However, for a complex $g^*(\cdot)$, this optimization problem may not
be easy to solve.  To simplify the optimization problem we can rely on
the smoothness of $g^*$ (with respect to a norm $\|\cdot\|_{D'}$) and
instead of directly maximizing the dual objective function, we try to
maximize the following proximal objective which is a lower bound of
the dual objective:
\begin{align*}
&~\argmax_{\Delta \alpha_i \in \reals^k} 
\left[ - \frac{1}{n} \phi^*_i(-(\alpha_i + \Delta
\alpha_i)) - \lambda \left(\nabla g^*(v^{(t-1)})^\top (\lambda n)^{-1}  X_i \Delta
\alpha_i +
\frac{1}{2} \| (\lambda n)^{-1}  X_i \Delta
\alpha_i\|_{D'}^2 \right) \right] \\
=&~\argmax_{\Delta \alpha_i \in \reals^k} 
\left[ -\phi^*_i(-(\alpha_i + \Delta
\alpha_i)) - w^{(t-1)\,\top} X_i \Delta
\alpha_i -
\frac{1}{2\lambda n} \| X_i \Delta
\alpha_i\|_{D'}^2
\right] .
\end{align*}
In general, this optimization problem is still not necessarily simple
to solve because $\phi^*$ may also be complex.  We will thus also
propose alternative update rules for $\Delta \alpha_i$ of the form
$\Delta \alpha_i = s (- \nabla \phi_i(X_i^\top w^{(t-1)}) -
\alpha_i^{(t-1)})$ for an appropriately chosen step size parameter
$s>0$.  Our analysis shows that an appropriate choice of $s$ still
leads to a sufficient increase in the dual objective.

It should be pointed out that we can always pick $\Delta \alpha_i$ 
so that the dual objective is non-decreasing. In fact, if 
for a specific choice of $\Delta \alpha_i$, the dual objective decreases,
we may simply set $\Delta \alpha_i=0$. 
Therefore throughout the proof we will assume that the dual objective
is non-decreasing whenever needed.

\begin{figure}[htbp]
\begin{myalgo}{Procedure Proximal Stochastic Dual Coordinate Ascent: Prox-SDCA($P,\epsilon,\alpha^{(0)}$)} 
\textbf{Goal:} Minimize $P(w) = \frac{1}{n} \sum_{i=1}^n
\phi_i(X_i^\top w) + \lambda g(w)$ \\
\textbf{Input:} Objective $P$, desired accuracy $\epsilon$, initial dual solution
$\alpha^{(0)}$ (default: $\alpha^{(0)}=0$) \\ 
\textbf{Assumptions:} \+ \\
$\forall i$, $\phi_i$ is $(1/\gamma)$-smooth w.r.t. $\|\cdot\|_P$ and
let $\|\cdot\|_D$ be the dual norm of $\|\cdot\|_P$ \\
 $g$ is $1$-strongly convex w.r.t. $\|\cdot\|_{P'}$ and
let $\|\cdot\|_{D'}$ be the dual norm of $\|\cdot\|_{P'}$ \\
 $\forall i$, $\|X_i\|_{D \to D'} \leq R$ \- \\
\textbf{Initialize} $v^{(0)}=\frac{1}{\lambda n} \sum_{i=1}^n X_i \alpha_i^{(0)}$, $w^{(0)}=\nabla g^*(0)$ \\
\textbf{Iterate:} for $t=1,2,\dots$ \+ \\
 Randomly pick $i$ \\
 Find $\Delta \alpha_i$ using any of the following options \+\+ \\
 (or any other update that achieves a larger dual objective): \- \\
 \textbf{Option I:} \+ \\ 
  $\displaystyle \Delta \alpha_i = \argmax_{\Delta \alpha_i}
\left[-\phi_i^*(-(\alpha_i^{(t-1)} + \Delta \alpha_i) ) - 
w^{(t-1)^\top} X_i \Delta \alpha_i - \frac{1}{2\lambda n} \left\| X_i
  \Delta \alpha_i \right\|_{D'}^2\right]$ \- \\
 \textbf{Option II:} \+ \\ 
  Let $u = - \nabla \phi_i(X_i^\top w^{(t-1)})$ and  $q = u- \alpha_i^{(t-1)} $ \\
   Let $\displaystyle s = \argmax_{s \in [0,1]} \left[-\phi_i^*(-(\alpha_i^{(t-1)} + sq) ) - 
s\,w^{(t-1)^\top} X_i q - \frac{s^2}{2\lambda n} \left\| X_i
  q \right\|_{D'}^2\right]$ \\ 
Set $\Delta \alpha_i = s q$ \- \\
 \textbf{Option III:} \+ \\
  Same as Option II but replace the definition of $s$ as follows:  \\
 Let $s = \min\left(1,\frac{\phi_i(X_i^\top w^{(t-1)})+\phi_i^*(-\alpha_i^{(t-1)})+ w^{(t-1)^\top} X_i \alpha^{(t-1)}_i
  + \frac{\gamma}{2} \|q\|_D^2}{ \|q\|_D^2 (\gamma + \frac{1}{\lambda n}
  \|X_i\|_{D \to D'}^2 )}\right)$ \-  \\
 \textbf{Option IV:} \+ \\
  Same as Option III but replace $\|X_i\|_{D \to D'}^2$ in the definition of $s$
  with $R^2$ \- \\
  \textbf{Option V:} \+ \\
  Same as Option II but replace the definition of $s$ to be $s =
  \frac{\lambda n \gamma}{R^2 + \lambda n
    \gamma}$ \-\- \\
  $\alpha^{(t)}_i \leftarrow \alpha^{(t-1)}_i + \Delta \alpha_i$ and
  for $j \neq i$, $\alpha^{(t)}_j \leftarrow \alpha^{(t-1)}_j$ \\
  $v^{(t)} \leftarrow v^{(t-1)} + (\lambda n)^{-1} X_i \Delta \alpha_i$ \\
  $w^{(t)} \leftarrow \nabla g^*(v^{(t)})$
  \- \\
  \textbf{Stopping condition}: \+ \\
  Let $T_0 < t$ (default: $T_0 = t - n - \lceil \frac{R^2}{\lambda
    \gamma} \rceil$ )\\
  \textbf{Averaging option:} \+ \\
  Let $\bar{\alpha} = \frac{1}{t-T_0} \sum_{i=T_0+1}^t \alpha^{(i-1)}$
  and $\bar{w} = \frac{1}{t-T_0} \sum_{i=T_0+1}^t
  w^{(i-1)}$ \- \\
  \textbf{Random option:} \+ \\
  Let $\bar{\alpha}=\alpha^{(i)}$ and $\bar{w} = w^{(i)}$ for some
  random $i \in T_0+1,\ldots,t$ \- \\
  Stop if $P(\bar{w})-D(\bar{\alpha}) \le \epsilon$ and output
  $\bar{w},\bar{\alpha},$ and $P(\bar{w})-D(\bar{\alpha})$
  \- \\
\end{myalgo}
\caption{The Generic Proximal Stochastic Dual Coordinate Ascent Algorithm}
\label{fig:sdca}
\end{figure}

The theorems below provide upper bounds on the number of iterations
required by our prox-SDCA procedure. 
\begin{theorem} \label{thm:smooth} 
  Consider Procedure Prox-SDCA as given in \figref{fig:sdca}.  Let
  $\alpha^*$ be an optimal dual solution and let
  $\epsilon > 0$. For every $T$ such that 
\[
T \geq \left(n +
  \frac{R^2}{\lambda \gamma}\right) \, \log\left( \left(n + \frac{R^2}{\lambda \gamma}\right)   \cdot \frac{D(\alpha^*)-D(\alpha^{(0)})}{\epsilon}\right) ,
\]
we are guaranteed that $\E [P(w^{(T)})-D(\alpha^{(T)})] \leq
\epsilon$.
Moreover, for every $T$ such that
\[
T \ge \left(n + \left\lceil \frac{R^2}{\lambda \gamma}\right\rceil
\right) \cdot \left(1 + 
\log\left(\frac{D(\alpha^*)-D(\alpha^{(0)})}{\epsilon}\right)
\right)
 ~,
\]
let $T_0 = T - n - \lceil \frac{R^2}{\lambda \gamma}\rceil$, then 
we are guaranteed that $\E [P(\bar{w})-D(\bar{\alpha})] \leq
\epsilon$.
\end{theorem}

We next give bounds that hold with high probability. 
\begin{theorem} \label{thm:HighProbsmooth} 
  Consider Procedure Prox-SDCA as given in \figref{fig:sdca}.  Let
  $\alpha^*$ be an optimal dual solution, 
let $\epsilon_D,\epsilon_P > 0$, and let $\delta \in (0,1)$. 
\begin{enumerate}
\item For every $T$ such that
\[
T \geq \left\lceil\left(n +
  \frac{R^2}{\lambda \gamma}\right) \,
\log\left(\frac{2(D(\alpha^*)-D(\alpha^{(0)}))}{\epsilon_D}\right)\right\rceil
\,\cdot\,\left\lceil \log_2\left(\frac{1}{\delta}\right)\right\rceil
 ~,
\]
we are guaranteed that with probability of at least $1-\delta$ it
holds that $D(\alpha^*)-D(\alpha^{(T)}) \le \epsilon_D$. 
\item For every $T$ such that
\[
T \geq \left\lceil\left(n +
  \frac{R^2}{\lambda \gamma}\right) \,\left(
\log\left(n + \frac{R^2}{\lambda \gamma}\right)  + \log\left( \frac{2(D(\alpha^*)-D(\alpha^{(0)}))}{\epsilon_P}\right)\right)\right\rceil
\,\cdot\,\left\lceil \log_2\left(\frac{1}{\delta}\right)\right\rceil
 ~,
\]
we are guaranteed that with probability of at least $1-\delta$ it
holds that $P(w^{(T)}) - D(\alpha^{(T)}) \le \epsilon_P$. 
\item Let $T$ be such that
\[
T \ge \left(n + \left\lceil \frac{R^2}{\lambda \gamma}\right\rceil
\right) \cdot \left(1 + \left\lceil 
\log\left(\frac{2(D(\alpha^*)-D(\alpha^{(0)}))}{\epsilon_P}\right)\right\rceil \right)
\,\cdot\,\left\lceil \log_2\left(\frac{2}{\delta}\right)\right\rceil 
 ~,
\]
and let $T_0 = T - n - \lceil \frac{R^2}{\lambda \gamma}\rceil$.
Suppose we choose $\lceil \log_2(2/\delta) \rceil$ values of $t$
uniformly at random from $T_0+1,\ldots,T$, and then choose the single
value of $t$ from these $\lceil \log_2(2/\delta) \rceil$ values for
which $P(w^{(t)})-D(\alpha^{(t)})$ is minimal. Then, with probability
of at least $1-\delta$ we have that $P(w^{(t)})-D(\alpha^{(t)}) \le
\epsilon_P$.
\end{enumerate}
\end{theorem}

The above theorem tells us that the runtime required to find an $\epsilon$ accurate
solution, with probability of at least $1-\delta$, is
\begin{equation} \label{eqn:totalRuntimeHighProb}
O\left(d\,  \left(n + \frac{R^2}{\lambda \gamma} \right) \cdot 
\log\left(\frac{D(\alpha^*)-D(\alpha^{(0)})}{\epsilon}\right)
\,\cdot\,\log\left(\frac{1}{\delta}\right) \right) ~.
\end{equation}

This yields the following corollary. 
\begin{corollary} \label{cor:smooth}
The expected runtime required to minimize $P$ up to accuracy
$\epsilon$ is 
\[
O\left(d\,  \left(n + \frac{R^2}{\lambda \gamma} \right) \cdot 
\log\left(\frac{D(\alpha^*)-D(\alpha^{(0)})}{\epsilon}\right)\right) ~.
\]
\end{corollary}
\begin{proof}
  We have shown that with a runtime of $O\left(d\, \left(n +
      \frac{R^2}{\lambda \gamma} \right) \cdot
    \log\left(\frac{2(D(\alpha^*)-D(\alpha^{(0)}))}{\epsilon}\right)\right)$
  we can find an $\epsilon$ accurate solution with probability of at
  least $1/2$. Therefore, we can run the procedure for this amount of
  time and check if the duality gap is smaller than $\epsilon$. If
  yes, we are done. Otherwise, we would restart the process. Since the
  probability of success is $1/2$ we have that the average number of
  restarts we need is $2$, which concludes the proof.
\end{proof}

\subsection{Acceleration} \label{sec:accelerate}

The Prox-SDCA procedure described in the previous subsection has the
iteration bound of $\tilde{O}\left(n + \tfrac{R^2}{\lambda
    \gamma}\right)$. This is a nearly linear runtime whenever the
condition number, $R^2/(\lambda \gamma)$, is $O(n)$. In this section
we show how to improve the dependence on the condition number by an
acceleration procedure. In particular, throughout this section we
assume that $10\,n < \frac{R^2}{\lambda \gamma}$. We further assume
throughout this subsection that the regularizer, $g$, is $1$-strongly
convex with respect to the Euclidean norm,
i.e. $\|u\|_{P'}=\|\cdot\|_2$. This also implies that $\|u\|_{D'}$ is
the Euclidean norm. A generalization of the acceleration technique for
strongly convex regularizers with respect to general norms is left to
future work.

The main idea of the acceleration procedure is to iteratively run the
Prox-SDCA procedure, where at iteration $t$ we call Prox-SDCA with the
modified objective, $\tilde{P}_t(w) = P(w) + \frac{\kappa}{2}
\|w-y^{(t-1)}\|^2$, where $\kappa$ is a relatively large regularization
parameter and the regularization is centered around the vector 
\[
y^{(t-1)} = w^{(t-1)} + \beta(w^{(t-1)}-w^{(t-2)}) 
\]
for some $\beta \in (0,1)$. That is, our regularization is centered
around the previous solution plus a ``momentum term''
$\beta(w^{(t-1)}-w^{(t-2)})$. 

A pseudo-code of the algorithm is given in \figref{fig:acc-SDCA}. 
Note that all the parameters of the algorithm are determined by our
theory.

\begin{figure}[htbp]
\begin{myalgo}{Procedure Accelerated Prox-SDCA}
\textbf{Goal:} Minimize $P(w) = \frac{1}{n} \sum_{i=1}^n
\phi_i(X_i^\top w) + \lambda g(w)$ \\
\textbf{Input:} Target accuracy $\epsilon$ (only used in the
stopping condition) \\
\textbf{Assumptions:} \+ \\
$\forall i$, $\phi_i$ is $(1/\gamma)$-smooth w.r.t. $\|\cdot\|_P$ and
let $\|\cdot\|_D$ be the dual norm of $\|\cdot\|_P$ \\
$g$ is $1$-strongly convex w.r.t. $\|\cdot\|_2$ \\
$\forall i$, $\|X_i\|_{D\to 2} \le R$ \\ 
$\frac{R^2}{\gamma \lambda} > 10\,n$ (otherwise, solve the problem using
vanilla Prox-SDCA) \- \\
\textbf{Define} $\kappa =\frac{R^2}{\gamma n}-\lambda$, 
$\mu = \lambda/2$, $\rho = \mu+\kappa$, 
$\eta= \sqrt{\mu/\rho}$, $\beta = \frac{1-\eta}{1+\eta}$,
\\
\textbf{Initialize}
$y^{(1)}=w^{(1)}=0$,~ $\alpha^{(1)}=0$,~$\xi_1 = (1+\eta^{-2})(P(0)-D(0))$ \\
\textbf{Iterate:} for $t=2,3,\ldots$ \+ \\
\textbf{Let} $\tilde{P}_t(w) = \frac{1}{n} \sum_{i=1}^n
\phi_i(X_i^\top w) + \tilde{\lambda} \tilde{g}_t(w)$ \+ \\
where $\tilde{\lambda} \tilde{g}_t(w) = \lambda g(w) +
\frac{\kappa}{2} \|w\|_2^2 - \kappa\, w^\top y^{(t-1)}$ \- \\
\textbf{Call} $(w^{(t)},\alpha^{(t)},\epsilon_t) =
\textrm{Prox-SDCA}\left(\tilde{P}_t,\frac{\eta}{2(1+\eta^{-2})}
\xi_{t-1},\alpha^{(t-1)}\right)$ \\
\textbf{Let} $y^{(t)} = w^{(t)} + \beta( w^{(t)} - w^{(t-1)})$ \\
\textbf{Let} $\xi_t = (1-\eta/2)^{t-1}\,\xi_1$ \\
\textbf{Stopping conditions:} break and return $w^{(t)}$ if one of the
following conditions hold: \+ \\
1.~~ $t \ge 1 + \frac{2}{\eta}
  \log(\xi_1/\epsilon) $ \\
2.~~  $(1+\rho/\mu)\epsilon_t + \frac{\rho\kappa}{2\mu} \|w^{(t)}-y^{(t-1)}\|^2 \le
\epsilon$ 
\- \\
\end{myalgo}
\caption{The Accelerated Prox-SDCA Algorithm}
\label{fig:acc-SDCA}
\end{figure}

\begin{remark}
  In the pseudo-code below, we specify the parameters based on our
  theoretical derivation. In our experiments, we found out that this
  choice of parameters also work very well in practice. However, we also
  found out that the algorithm is not very sensitive to the choice of
  parameters. For example, we found out that running $5n$ iterations
  of Prox-SDCA (that is, $5$ epochs over the data), without checking
  the stopping condition, also works very well.
\end{remark}

The main theorem is the following.
\begin{theorem} \label{thm:acceleratedThmMain}
Consider the accelerated Prox-SDCA algorithm given in
\figref{fig:acc-SDCA}. 
\begin{itemize}
\item Correctness: When the algorithm terminates we have that $P(w^{(t)})-P(w^*)\le
\epsilon$. 
\item Runtime: 
\begin{itemize}
\item The number of outer iterations is at most 
\[
1 + \frac{2}{\eta} \log(\xi_1/\epsilon) ~\le~
1 + \sqrt{\frac{8R^2}{\lambda \gamma n}}
\left(\log\left(\frac{2R^2}{\lambda \gamma n}\right) +
  \log\left(\frac{P(0)-D(0)}{\epsilon}\right) \right) ~.
\]
\item 
Each outer iteration involves a single call to Prox-SDCA,
and the averaged runtime required by each such call is
\[
O\left(d\,n \log\left(\frac{R^2}{\lambda \gamma n}\right)\right) ~.
\]
\end{itemize}
\end{itemize}
\end{theorem}

By a straightforward amplification argument we obtain that for every
$\delta \in (0,1)$ the total
runtime required by accelerated Prox-SDCA to guarantee an
$\epsilon$-accurate solution with probability of at least $1-\delta$
is 
\[
O\left(d\,\sqrt{\frac{nR^2}{\lambda\,\gamma}}\,
  \log\left(\frac{R^2}{\lambda\,\gamma\,n}\right)\,\left(\log\left(\frac{R^2}{\lambda \gamma n}\right) +
  \log\left(\frac{P(0)-D(0)}{\epsilon}\right) \right) \, \log\left(\frac{1}{\delta}\right)
\right) ~.
\]

\subsection{Non-smooth, Lipschitz, loss functions} \label{sec:Lipschitz}

So far we have assumed that for every $i$, $\phi_i$ is a
$(1/\gamma)$-smooth function. We now consider the case in which
$\phi_i$ might be non-smooth, and even non-differentiable, but it is
$L$-Lipschitz. 

Following \citet{nesterov2005smooth}, we apply a ``smoothing'' technique.
We first observe that if $\phi$ is $L$-Lipschitz function then the
domain of $\phi^*$ is in the ball of radius $L$.
\begin{lemma} \label{lem:LipConjDom}
Let $\phi : \reals^k \to \reals$ be an $L$-Lipschitz function w.r.t. a
norm $\|\cdot\|_P$ and let $\|\cdot\|_D$ be the dual norm. Then,
for any $\alpha \in \reals^k$ s.t. $\|\alpha\|_D > L$ we have that $\phi^*(\alpha) =
\infty$. 
\end{lemma}
\begin{proof}
  Fix some $\alpha$ with $\|\alpha\|_D > L$. Let $x_0$ be a vector
  such that $\|x_0\|_P = 1$ and $\alpha^\top x_0 = \|\alpha\|_D$ (this is a
  vector that achieves the maximal objective in the definition of the
  dual norm). By definition of
  the conjugate we have
\begin{align*}
\phi^*(\alpha)  &= \sup_x [\alpha^\top\,x - \phi(x)] \\
&= -\phi(0) + \sup_{x } [\alpha^\top\,x - (\phi(x) - \phi(0))] \\
&\ge -\phi(0) + \sup_{x } [\alpha^\top\,x - L \|x-0\|_P] \\
&\ge -\phi(0) + \sup_{c > 0 } [\alpha^\top\,(cx_0) - L \|cx_0\|_P] \\
&= -\phi(0) + \sup_{c > 0} (\|\alpha\|_D-L)\,c = \infty ~.
\end{align*}
\end{proof}

This observation allows us to smooth $L$-Lipschitz functions by adding
regularization to their conjugate. In particular, the following lemma
generalizes Lemma 2.5 in \cite{shalev2010trading}.
\begin{lemma} \label{lem:smoothingLemma}
  Let $\phi$ be a proper, convex, $L$-Lipschitz function w.r.t. a norm
  $\|\cdot\|_P$, let $\|\cdot\|_D$ be the dual norm, and let $\phi^*$
  be the conjugate of $\phi$. Assume that $\|\cdot\|_2 \le
  \|\cdot\|_D$. Define $\tilde{\phi}^*(\alpha) = \phi^*(\alpha) +
  \frac{\gamma}{2} \|\alpha\|_2^2$ and let $\tilde{\phi}$ be the
  conjugate of $\tilde{\phi}^*$. Then, $\tilde{\phi}$ is
  $(1/\gamma)$-smooth w.r.t. the Euclidean norm and
\[
\forall a,~~ 0 \le \phi(a) - \tilde{\phi}(a) \le \gamma L^2/2 ~.
\]
\end{lemma}
\begin{proof}
The fact that $\tilde{\phi}$ is $(1/\gamma)$-smooth follows directly from
the fact that $\tilde{\phi}^*$ is $\gamma$-strongly convex. 
For the second claim note that 
\[
\tilde{\phi}(a) = \sup_b  \left[b a - \phi^*(b) - \frac{\gamma}{2} \|b\|_2^2
\right] \le \sup_b  \left[b a - \phi^*(b) \right] = \phi(a) 
\]
and
\begin{align*}
\tilde{\phi}(a) &= \sup_b  \left[b a - \phi^*(b) - \frac{\gamma}{2} \|b\|_2^2
\right] 
= \sup_{b:\|b\|_D \le L}  \left[b a - \phi^*(b) - \frac{\gamma}{2} \|b\|_2^2
\right] \\
&\ge \sup_{b:\|b\|_D \le L}  \left[b a - \phi^*(b) - \frac{\gamma}{2} \|b\|_D^2
\right] 
\ge \sup_{b:\|b\|_D \le L}  \left[b a - \phi^*(b) \right] - \frac{\gamma}{2} L^2 \\
&= \phi(a) - \frac{\gamma}{2} L^2 ~.
\end{align*}
\end{proof}

\begin{remark}
It is also possible to smooth using different regularization functions
which are strongly convex with respect to other norms. See
\citet{nesterov2005smooth} for discussion.
\end{remark}

\section{Proof of \thmref{thm:acceleratedThmMain}} \label{sec:acceleratedThmMain}

The first claim of the theorem is that when the procedure stops we
have $ P(w^{(t)})-P(w^*) \le \epsilon$. We therefore need to show that
each stopping condition guarantees that $ P(w^{(t)})-P(w^*) \le
\epsilon$.


For the second stopping condition, recall that $w^{(t)}$ is an
$\epsilon_t$-accurate minimizer of $P(w) + \frac{\kappa}{2}
\|w-y^{(t-1)}\|^2$, and hence by \lemref{lem:lowerBoundQ}
below (with $z=w^*$, $w^+=w^{(t)}$, and $y=y^{(t-1)}$):
\begin{align*}
P(w^*) &\ge P(w^{(t)}) + Q_{\epsilon}(w^*;w^{(t)},y^{(t-1)}) \\
&\ge P(w^{(t)})  - 
\frac{\rho\kappa}{2\mu}\|y^{(t-1)}-w^{(t)}\|^2 - (1+\rho/\mu)\epsilon_t ~.
\end{align*}

It is left to show that the first stopping condition is correct,
namely, to show that after $1 + \frac{2}{\eta} \log(\xi_1/\epsilon)$
iterations the algorithm must converge to an $\epsilon$-accurate
solution. Observe that the definition of $\xi_t$ yields that $ \xi_{t}
= (1-\eta/2)^{t-1} \, \xi_1 \le e^{-\eta(t-1)/2} \xi_1$.  Therefore,
to prove that the first stopping condition is valid, it suffices to
show that for every $t$, $ P(w^{(t)})-P(w^*) \le \xi_t$.

Recall that at each outer iteration of the accelerated procedure, we approximately
minimize an objective of the form
\[
P(w;y) = P(w) + \frac{\kappa}{2} \|w-y\|^2 ~.
\]
Of course, minimizing $P(w;y)$ is not the same as minimizing
$P(w)$. Our first lemma shows that for every $y$, if $w^+$ is an
$\epsilon$-accurate minimizer of $P(w;y)$ then we can derive a lower
bound on $P(w)$ based on $P(w^+)$ and a convex quadratic function of $w$.
\begin{lemma} \label{lem:lowerBoundQ}
Let $\mu=\lambda/2$ and $\rho = \mu+\kappa$. 
Let $w^+$ be a vector such that
$P(w^+;y) \le \min_w P(w,y) + \epsilon$. Then, for every $z$, 
\[
P(z) \ge P(w^+) +  Q_{\epsilon}(z;w^+,y) ~,
\]
where
\[
Q_{\epsilon}(z;w^+,y) = \frac{\mu}{2} \left\| z - \left(y -
      \tfrac{\rho}{\mu}(y-w^+)\right)\right\|^2 - \frac{\rho\kappa}{2\mu}\|y-w^+\|^2 - (1+\rho/\mu)\epsilon ~.
\]
\end{lemma}
\begin{proof}
Denote 
\[
\Psi(w) = P(w) - \frac{\mu}{2} \|w\|^2 ~.
\] 
We can write
\[
\frac{1}{2} \|w\|^2 = \frac{1}{2} \|y\|^2 + y^\top (w-y) + \frac{1}{2}
\|w-y\|^2 ~.
\]
It follows that 
\[
P(w) = \Psi(w) + \frac{\mu}{2} \|w\|^2 = \Psi(w) +  \frac{\mu}{2} \|y\|^2 + \mu\,y^\top (w-y) + \frac{\mu}{2}
\|w-y\|^2 ~.
\]
Therefore, we can rewrite $P(w;y)$ as:
\[
P(w;y) = \Psi(w) + \frac{\mu}{2} \|y\|^2 + \mu \,y^\top (w-y) + \frac{\rho}{2} \|w-y\|^2 ~.
\]
Let $\tilde{w} = \argmin_w P(w;y)$. Therefore,
the gradient\footnote{If the regularizer $g(w)$ in the definition of
  $P(w)$ is non-differentiable, we can replace $\nabla
  \Psi(\tilde{w})$ with an appropriate sub-gradient of $\Psi$ at
  $\tilde{w}$. It is easy to verify that the proof is still valid.}
of $P(w;y)$ w.r.t. $w$ vanishes at $\tilde{w}$, which yields
\[
\nabla \Psi(\tilde{w}) + \mu y + \rho(\tilde{w}-y) = 0 ~~\Rightarrow~~ 
\nabla \Psi(\tilde{w}) + \mu y = \rho(y-\tilde{w}) 
~.
\]
By the $\mu$-strong convexity of $\Psi$ we have that for every $z$,
\[
\Psi(z) \ge \Psi(\tilde{w}) + \nabla \Psi(\tilde{w})^\top (z-\tilde{w}) +
\frac{\mu}{2}\|z-\tilde{w}\|^2 ~.
\]
Therefore, 
\begin{align*}
P(z) &=  \Psi(z) + \frac{\mu}{2} \|y\|^2 + \mu \,y^\top (z-y) +
\frac{\mu}{2} \|z-y\|^2 \\
&\ge \Psi(\tilde{w}) + \nabla \Psi(\tilde{w})^\top (z-\tilde{w}) +
\frac{\mu}{2}\|z-\tilde{w}\|^2 +  \frac{\mu}{2} \|y\|^2 + \mu \,y^\top (z-y) +
\frac{\mu}{2} \|z-y\|^2 \\
&= P(\tilde{w};y) - \frac{\rho}{2}\|\tilde{w}-y\|^2+ \nabla \Psi(\tilde{w})^\top (z-\tilde{w}) 
+ \mu \,y^\top (z-\tilde{w})  + \frac{\mu}{2}\left(\|z-\tilde{w}\|^2 +\|z-y\|^2 \right) \\
&= P(\tilde{w};y) - \frac{\rho}{2}\|\tilde{w}-y\|^2 + \rho(y-\tilde{w})^\top(z-\tilde{w}) + \frac{\mu}{2}\left(\|z-\tilde{w}\|^2 +\|z-y\|^2 \right) \\
&= P(\tilde{w};y) + \frac{\rho}{2}\|\tilde{w}-y\|^2 + \rho(y-\tilde{w})^\top(z-y) +
\frac{\mu}{2}\left(\|z-\tilde{w}\|^2 + \|z-y\|^2 \right) ~.
\end{align*}
In addition, by standard algebraic manipulations,
\begin{align*}
&\frac{\rho}{2}\|\tilde{w}-y\|^2 + \rho(y-\tilde{w})^\top(z-y) +
\frac{\mu}{2}\|z-\tilde{w}\|^2 - \left( \frac{\rho}{2}\|w^+-y\|^2 + \rho(y-w^+)^\top(z-y) +
\frac{\mu}{2}\|z-w^+\|^2\right)\\
&= \left(\rho(w^+-y)-\rho(z-y)+\mu(w^+-z)\right)^\top(\tilde{w}-w^+) +
\frac{\rho+\mu}{2} \|\tilde{w}-w^+\|^2\\
&= (\rho+\mu)(w^+-z)^\top(\tilde{w}-w^+) +
\frac{\rho+\mu}{2} \|\tilde{w}-w^+\|^2\\
&= \frac{1}{2}\left\| \sqrt{\mu}(w^+-z) + \frac{\rho+\mu}{\sqrt{\mu}} (\tilde{w}-w^+)
\right\|^2 
- \frac{\mu}{2} \|z-w^+\|^2 - \frac{(\rho+\mu)^2}{2\mu}\|\tilde{w}-w^+\|^2  +
\frac{\rho+\mu}{2} \|\tilde{w}-w^+\|^2
\\
&\ge - \frac{\mu}{2} \|z-w^+\|^2 - \frac{\rho(\rho+\mu)}{2\mu}\|\tilde{w}-w^+\|^2
~.
\end{align*}
Since $P(\cdot;y)$ is $(\rho+\mu)$-strongly
convex and $\tilde{w}$ minimizes
$P(\cdot;y)$, we have that for every $w^+$ it holds that
$\frac{\rho+\mu}{2} \|\tilde{w}-w^+\|^2 \le
P(w^+;y)-P(\tilde{w};y)$. Combining all the above and using
the fact that for every $w,y$, $P(w;y) \ge P(w)$, we obtain that for every $w^+$,
\[
P(z) \ge P(w^+) + \frac{\rho}{2}\|w^+-y\|^2 +
\rho(y-w^+)^\top(z-y) + \frac{\mu}{2}\|z-y\|^2 - \left(1+\frac{\rho}{\mu}\right)\left(P(w^+;y)-P(\tilde{w};y)\right) ~.
\]
Finally, using the assumption $P(w^+;y) \le
\min_w P(w;y) + \epsilon$ we conclude our proof. 
\end{proof}

We saw that the quadratic function $P(w^+) + Q_{\epsilon}(z;w^+,y)$ lower
bounds the function $P$ everywhere. Therefore, any convex combination of
such functions would form a quadratic function which 
lower bounds $P$. In particular, the algorithm (implicitly)
maintains a sequence of quadratic functions, $h_1,h_2,\ldots$, defined
as follows. Choose  $\eta \in (0,1)$ and a sequence $y^{(1)},y^{(2)},\ldots$ that will be
specified later. Define, 
\[
h_1(z) = P(0) + Q_{P(0)-D(0)}(z;0,0) = 
P(0) + \frac{\mu}{2}\|z\|^2 - (1+\rho/\mu)(P(0)-D(0)) ~,
\]
and for $t \ge 1$, 
\[
  h_{t+1}(z) = (1-\eta) h_{t}(z) + \eta (P(w^{(t+1)}) + Q_{\epsilon_{t+1}}(z;w^{(t+1)},y^{(t)})) ~.
\]
The following simple lemma shows that for every $t \ge 1$ and $z$, $h_t(z)$ lower
bounds $P(z)$. 

\begin{lemma} \label{lem:boundPbyh} Let $\eta \in (0,1)$ and let
  $y^{(1)},y^{(2)},\ldots$ be any sequence of vectors. Assume that
  $w^{(1)}=0$ and for
  every $t \ge 1$, $w^{(t+1)}$ satisfies $P(w^{(t+1)};y^{(t)}) \le \min_w
  P(w;y^{(t)}) + \epsilon_{t+1}$. Then, for every $t \ge 1$ and every vector
  $z$ we have
\[
h_t(z) \le P(z) ~.
\]
\end{lemma}
\begin{proof}
  The proof is by induction. For $t=1$, observe that $P(0;0) = P(0)$
  and that for every $w$ we have $P(w;0) \ge P(w) \ge D(0)$. This
  yields $P(0;0) - \min_w P(w;0) \le P(0)-D(0)$. The claim now follows
  directly from \lemref{lem:lowerBoundQ}. Next, for the inductive
  step, assume the claim holds for some $t-1 \ge 1$ and let us prove
  it for $t$. By the recursive definition of $h_t$ and by using
  \lemref{lem:lowerBoundQ} we have
\[
 h_t(z) = (1-\eta) h_{t-1}(z) + \eta
 (P(w^{(t)}) + Q_{\epsilon_{t}}(z;w^{(t)},y^{(t-1)}))  \le  (1-\eta) h_{t-1}(z) +
\eta P(z) ~.
\]
Using the inductive assumption we obtain that the right-hand side of
the above is upper bounded by $(1-\eta)P(z)+\eta P(z) = P(z)$,
which concludes our proof. 
\end{proof}

The more difficult part of the proof is to show that for every $t \ge
1$, 
\[
P(w^{(t)}) \le \min_w h_t(w) + \xi_t ~.
\]
If this holds true, then we would immediately get that for every
$w^*$,
\[
P(w^{(t)})-P(w^*) \le P(w^{(t)}) - h_t(w^*) \le 
P(w^{(t)}) - \min_w h_t(w) \le \xi_t ~.
\]
This will conclude the proof of the first part of
\thmref{thm:acceleratedThmMain}, since $\xi_t = \xi_1
(1-\eta/2)^{t-1} \le \xi_1\,e^{-(t-1)\eta/2}$, and therefore, $1 +
\frac{2}{\eta} \log(\xi_1/\epsilon)$ iterations suffice to
guarantee that $P(w^{(t)})-P(w^*) \le \epsilon$. 

Define
\[
v^{(t)} = \argmin_w h_t(w) ~.
\]
Let us construct an explicit formula for $v^{(t)}$. Clearly, $v^{(1)} = 0$.
Assume that we have calculated $v^{(t)}$ and let us calculate $v^{(t+1)}$. Note
that $h_t$ is a quadratic function which is minimized at
$v^{(t)}$. Furthermore, it is easy to see that for every $t$, $h_t$ is
$\mu$-strongly convex quadratic function. Therefore,
\[
h_t(z) = h_t(v^{(t)}) + \frac{\mu}{2} \|z - v^{(t)}\|^2 ~.
\]
By the definition of $h_{t+1}$ we obtain that 
\[
h_{t+1}(z) = (1-\eta) (h_t(v^{(t)}) + \frac{\mu}{2} \|z - v^{(t)}\|^2  ) + \eta (P(w^{(t+1)}) + Q_{\epsilon_{t+1}}(z;w^{(t+1)},y^{(t)})) ~.
\]
Since the gradient of $h_{t+1}(z)$ at $v^{(t+1)}$ should be zero, we
obtain that $v^{(t+1)}$ should satisfy 
\[
(1-\eta)\mu(v^{(t+1)} -v^{(t)}) + \eta \mu \left( v^{(t+1)} - (y^{(t)} -
\tfrac{\rho}{\mu}(y^{(t)}-w^{(t+1)}) ) \right) = 0
\]
Rearranging, we obtain
\begin{equation} \label{eqn:vtp1exp}
v^{(t+1)} = (1-\eta)v^{(t)} + \eta (y^{(t)} -
\tfrac{\rho}{\mu}(y^{(t)}-w^{(t+1)}) ) ~.
\end{equation}

Getting back to our second phase of the proof, we need to show that
for every $t$ we have $P(w^{(t)}) \le h_t(v^{(t)}) + \xi_t$. We do so
by induction. For the case $t=1$ we have
\[
P(w^{(1)}) - h_1(v^{(1)}) = P(0) - h_1(0) = (1+\rho/\mu)(P(0)-D(0)) =
\xi_1 ~.
\]
For the induction step, assume the claim holds for $t \ge 1$ and let
us prove it for $t+1$. We use the shorthands,
\[
Q_t(z) = Q_{\epsilon_t}(z;w^{(t)},y^{(t-1)}) ~~~\text{and}~~~
\psi_t(z) = Q_t(z) + P(w^{(t)}) ~~.
\]
Let us rewrite $h_{t+1}(v^{(t+1)})$ as
\begin{align*}
h_{t+1}(v^{(t+1)}) &= (1-\eta)h_t(v^{(t+1)}) + \eta \psi_{t+1}(v^{(t+1)}) \\
&= (1-\eta)(h_t(v^{(t)})+\frac{\mu}{2} \|v^{(t)}-v^{(t+1)}\|^2) +
\eta \psi_{t+1}(v^{(t+1)}) ~.
\end{align*}
By the inductive assumption we have $h_t(v^{(t)}) \ge P(w^{(t)})
-\xi_t$ and by \lemref{lem:lowerBoundQ} we have $P(w^{(t)}) \ge
\psi_{t+1}(w^{(t)})$. Therefore, 
\begin{align} \label{eqn:inductivebeg}
h_{t+1}(v^{(t+1)}) &\ge  (1-\eta)(\psi_{t+1}(w^{(t)})-\xi_t+\frac{\mu}{2} \|v^{(t)}-v^{(t+1)}\|^2) +
\eta \psi_{t+1}(v^{(t+1)}) \\ \nonumber
&= \frac{(1-\eta)\mu}{2} \|v^{(t)}-v^{(t+1)}\|^2 + \eta \psi_{t+1}(v^{(t+1)}) +
(1-\eta)\psi_{t+1}(w^{(t)}) - (1-\eta)\xi_t  ~.
\end{align}
Next, note that we can rewrite
\[
Q_{t+1}(z) =\frac{\mu}{2} \|z-y^{(t)}\|^2 +
\rho(z-y^{(t)})^\top(y^{(t)}-w^{(t+1)}) + \frac{\rho}{2}
\|y^{(t)}-w^{(t+1)}\|^2 - (1+\rho/\mu)\epsilon_{t+1}~.
\]
Therefore,
\begin{align} \label{eqn:inductivepsi}
&\eta \psi_{t+1}(v^{(t+1)}) +
(1-\eta)\psi_{t+1}(w^{(t)}) - P(w^{(t+1)}) +
(1+\rho/\mu)\epsilon_{t+1}\\ \nonumber
&= \frac{\eta\mu}{2}
\|v^{(t+1)}-y^{(t)}\|^2 +\frac{(1-\eta)\mu}{2}
\|w^{(t)}-y^{(t)}\|^2 + \rho(\eta v^{(t+1)} +
(1-\eta)w^{(t)}-y^{(t)})^\top(y^{(t)}-w^{(t+1)}) \\ \nonumber
&+ \frac{\rho}{2}
\|y^{(t)}-w^{(t+1)}\|^2
\end{align}
So far we did not specify $\eta$ and $y^{(t)}$ (except $y^{(0)}=0$). We next set 
\[
\eta = \sqrt{\mu/\rho} ~~~\text{and}~~~ \forall t \ge 1,~y^{(t)} = (1+\eta)^{-1}(\eta
v^{(t)} + w^{(t)}) ~.
\]
This choices guarantees that (see \eqref{eqn:vtp1exp}) 
\begin{align*}
\eta v^{(t+1)} + (1-\eta)w^{(t)} &=
\eta(1-\eta)v^{(t)}+\eta^2(1-\frac{\rho}{\mu})y^{(t)} + \eta^2
\frac{\rho}{\mu} w^{(t+1)} +
(1-\eta)w^{(t)} \\
&= w^{(t+1)} + (1-\eta) \left[ \eta v^{(t)} +
  \frac{\eta^2(1-\frac{\rho}{\mu})}{1-\eta}y^{(t)} + w^{(t)}
\right] \\
&= w^{(t+1)} + (1-\eta) \left[ \eta v^{(t)} -
  \frac{1-\eta^2}{1-\eta}y^{(t)} + w^{(t)} \right] \\
&= w^{(t+1)} + (1-\eta) \left[ \eta v^{(t)} -
  (1+\eta)y^{(t)} + w^{(t)} \right] \\
&= w^{(t+1)} ~.
\end{align*}
We also observe that $\epsilon_{t+1} \le \frac{\eta
  \xi_t}{2(1+\eta^{-2})}$ which implies that 
$ (1+\rho/\mu)\epsilon_{t+1} +
(1-\eta) \xi_t \le (1-\eta/2)\xi_t = \xi_{t+1}$. 
Combining the above with \eqref{eqn:inductivebeg} and
\eqref{eqn:inductivepsi}, and rearranging terms, we obtain that
\begin{align*}
&h_{t+1}(v^{(t+1)}) - P(w^{(t+1)}) + \xi_{t+1} - \frac{(1-\eta)\mu}{2}
\|w^{(t)}-y^{(t)}\|^2\\
&\ge \frac{(1-\eta)\mu}{2} \|v^{(t)}-v^{(t+1)}\|^2
+ \frac{\eta\mu}{2}
\|v^{(t+1)}-y^{(t)}\|^2 - \frac{\rho}{2}
\|y^{(t)}-w^{(t+1)}\|^2  ~.
\end{align*}
Next, observe that $\rho \eta^2= \mu$ and that by
\eqref{eqn:vtp1exp} we have
\[
y^{(t)}-w^{(t+1)} = \eta\left[ \eta y^{(t)} + (1-\eta)v^{(t)} -
  v^{(t+1)}\right] ~.
\]
We therefore obtain that
\begin{align*}
&h_{t+1}(v^{(t+1)}) - P(w^{(t+1)}) + \xi_{t+1} - \frac{(1-\eta)\mu}{2}
\|w^{(t)}-y^{(t)}\|^2 \\
&\ge \frac{(1-\eta)\mu}{2} \|v^{(t)}-v^{(t+1)}\|^2
+ \frac{\eta\mu}{2}
\|y^{(t)} -v^{(t+1)}\|^2  - \frac{\mu}{2}
\|\eta y^{(t)} + (1-\eta)v^{(t)} -
  v^{(t+1)} \|^2 ~.
\end{align*}
The right-hand side of the above is non-negative because of
the convexity of the function $f(z) = \frac{\mu}{2}
\|z-v^{(t+1)}\|^2$, which yields
\[
P(w^{(t+1)}) \le h_{t+1}(v^{(t+1)}) + \xi_{t+1} - \frac{(1-\eta)\mu}{2}
\|w^{(t)}-y^{(t)}\|^2 \le h_{t+1}(v^{(t+1)}) + \xi_{t+1} ~.
\]
This concludes our inductive argument. 

\paragraph{Proving the ``runtime'' part of
  \thmref{thm:acceleratedThmMain}:}

We next show that each call to Prox-SDCA will terminate quickly. 
By the definition of $\kappa$ we have that 
\[
\frac{R^2}{(\kappa+\lambda)\gamma} = n ~.
\]
Therefore, based
on \corref{cor:smooth} we know that the averaged runtime at iteration $t$ is 
\[
O\left(d\,n \log\left( \frac{\tilde{D}_{t}(\alpha^*)-\tilde{D}_{t}(\alpha^{(t-1)})}{\frac{\eta}{2(1+\eta^{-2})}\xi_{t-1}} \right)\right) ~.
\]
The following lemma bounds the initial dual sub-optimality at
iteration $t \ge 4$. Similar arguments will yield a similar result for
$t < 4$. 
\begin{lemma}
\[
\tilde{D}_{t}(\alpha^*)-\tilde{D}_{t}(\alpha^{(t-1)}) \le \epsilon_{t-1} + \frac{36\kappa}{\lambda}  \xi_{t-3} ~.
\]
\end{lemma}
\begin{proof}
  Define $\tilde{\lambda} = \lambda + \kappa$, $f(w) =
  \frac{\lambda}{\tilde{\lambda}}g(w) +
  \frac{\kappa}{2\tilde{\lambda}}\|w\|^2$, and $\tilde{g}_t(w) = f(w)
  - \frac{\kappa}{\tilde{\lambda}} w^\top y^{(t-1)}$. Note that
  $\tilde{\lambda}$ does not
  depend on $t$ and therefore $v(\alpha) = \frac{1}{n \tilde{\lambda}}
  \sum_i X_i \alpha_i$ is the same for every $t$. Let, 
\[
\tilde{P}_t(w) = \frac{1}{n} \sum_{i=1}^n \phi_i(X_i^\top w) +  \tilde{\lambda}
\tilde{g}_t(w) ~.
\]
We have
\begin{equation} \label{eqn:tildePbso}
\tilde{P}_t(w^{(t-1)}) = \tilde{P}_{t-1}(w^{(t-1)}) +
\kappa w^{(t-1) \top} (y^{(t-2)}-y^{(t-1)}) ~.
\end{equation}
Since
\[
\tilde{g}_t^*(\theta) = \max_w w^\top
(\theta+\frac{\kappa}{\tilde{\lambda}} y^{(t-1)}) - f(w) 
= f^*(\theta+\frac{\kappa}{\tilde{\lambda}} y^{(t-1)}) ~,
\]
we obtain that the dual problem is
\[
\tilde{D}_t(\alpha) = -\frac{1}{n} \sum_i \phi^*_i(-\alpha_i) -
\tilde{\lambda} f^*(v(\alpha) +\frac{\kappa}{\tilde{\lambda}} y^{(t-1)})
\]
Let $z = \frac{\kappa}{\tilde{\lambda}} (y^{(t-1)} - y^{(t-2)})$,
then, by the smoothness of $f^*$ we have 
\[
f^*(v(\alpha) +\frac{\kappa}{\tilde{\lambda}} y^{(t-1)}) =
f^*(v(\alpha) +\frac{\kappa}{\tilde{\lambda}} y^{(t-2)} + z) \le
f^*(v(\alpha) +\frac{\kappa}{\tilde{\lambda}} y^{(t-2)}) + \nabla f^*(v(\alpha) +\frac{\kappa}{\tilde{\lambda}} y^{(t-2)})^\top  z
+ \frac{1}{2} \|z\|^2 ~.
\]
Applying this for $\alpha^{(t-1)}$ and using $w^{(t-1)} = \nabla
\tilde{g}_{t-1}^*(v(\alpha^{(t-1)})) = \nabla f^*(v(\alpha^{(t-1)}) +
\frac{\kappa}{\tilde{\lambda}}y^{(t-2)})$, we obtain 
\[
f^*(v(\alpha^{(t-1)} ) +\frac{\kappa}{\tilde{\lambda}} y^{(t-1)}) \le 
f^*(v(\alpha^{(t-1)} ) +\frac{\kappa}{\tilde{\lambda}} y^{(t-2)}) + w^{(t-1)\,\top}  z
+ \frac{1}{2} \|z\|^2 ~.
\]
It follows that
\[
-\tilde{D}_t(\alpha^{(t-1)}) + \tilde{D}_{t-1}(\alpha^{(t-1)}) \le
\kappa w^{(t-1) \top} (y^{(t-1)}-y^{(t-2)})  +
\frac{\kappa^2}{2\tilde{\lambda}} \|y^{(t-1)}-y^{(t-2)}\|^2 ~. 
\]
Combining the above with \eqref{eqn:tildePbso}, we obtain that
\[
\tilde{P}_t(w^{(t-1)}) - \tilde{D}_t(\alpha^{(t-1)}) \le
\tilde{P}_{t-1}(w^{(t-1)}) - \tilde{D}_{t-1}(\alpha^{(t-1)}) + \frac{\kappa^2}{2 \tilde{\lambda}} \|y^{(t-1)}-y^{(t-2)}\|^2 ~. 
\]
Since $\tilde{P}_t(w^{(t-1)}) \ge \tilde{D}_t(\alpha^*)$ and since
$\tilde{\lambda} \ge \kappa$ we get that
\[
\tilde{D}_t(\alpha^*) - \tilde{D}_t(\alpha^{(t-1)}) \le
\epsilon_{t-1} + \frac{\kappa}{2} \|y^{(t-1)}-y^{(t-2)}\|^2 ~. 
\]
Next, we bound $\|y^{(t-1)}-y^{(t-2)}\|^2$. We have
\begin{align*}
\|y^{(t-1)}-y^{(t-2)}\| &= \|w^{(t-1)} - w^{(t-2)} + \beta(w^{(t-1)}-w^{(t-2)} -
w^{(t-2)} + w^{(t-3)})\| \\
&\le 3 \max_{i \in \{1,2\}} \|w^{(t-i)} -
w^{(t-i-1)}\|  ~,
\end{align*}
where we used the triangle inequality and $\beta < 1$.
By strong convexity of $P$ we have, for every $i$, 
\[
\|w^{(i)}-w^*\| \le \sqrt{\frac{P(w^{(i)})-P(w^*)}{\lambda/2}} \le
\sqrt{\frac{\xi_i}{\lambda/2}} ~,
\]
which implies
\[
\|w^{(t-i)}-w^{(t-i-1)}\| \le \|w^{(t-i)}-w^*\| + \|w^* - w^{(t-i-1)}\|
\le 2 \sqrt{\frac{\xi_{t-i-1}}{\lambda/2}} ~.
\]
This yields the bound
\[
\|y^{(t-1)}-y^{(t-2)}\|^2 \le 72
  \frac{\xi_{t-3}}{\lambda} ~.
\]
All in all, we have obtained that
\[
\tilde{D}_t(\alpha^*) - \tilde{D}_t(\alpha^{(t-1)}) \le
\epsilon_{t-1} + \frac{36\kappa}{\lambda}  \xi_{t-3} ~.
\]
\end{proof}

Getting back to the proof of the second claim of
\thmref{thm:acceleratedThmMain}, we have obtained that 
\begin{align*}
\frac{\tilde{D}_{t}(\alpha^*)-\tilde{D}_{t}(\alpha^{(t-1)})}{\frac{\eta}{2(1+\eta^{-2})}\xi_{t-1}}
&\le \frac{\epsilon_{t-1}}{\frac{\eta}{2(1+\eta^{-2})}\xi_{t-1}} + \frac{36\kappa
  \xi_{t-3}}{\lambda \frac{\eta}{2(1+\eta^{-2})}\xi_{t-1}} \\
&\le  (1-\eta/2)^{-1} + 
 \frac{36\kappa 2(1+\eta^{-2})
  }{\lambda \eta} (1-\eta/2)^{-2} \\
&\le (1-\eta/2)^{-4}\left(1 + \frac{72\kappa (1+\eta^{-2})
  }{\lambda \eta} \right) \\
&\le (1-\eta/2)^{-2}\left(1 + 36\eta^{-5}\right) ~~,
\end{align*}
where in the last inequality we used $\eta^{-2} - 1 =
\frac{2\kappa}{\lambda}$, which implies that
$\frac{2\kappa}{\lambda}(1+\eta^{-2}) \le \eta^{-4}$. Using $1 <
\eta^{-5}$, $1-\eta/2 \ge 0.5$, and taking log to both
sides, we get that
\[
\log\left(\frac{\tilde{D}_{t}(\alpha^*)-\tilde{D}_{t}(\alpha^{(t-1)})}{\frac{\eta}{2(1+\eta^{-2})}\xi_{t-1}}\right) 
\le 2\log(2) + \log(37) - 5 \log(\eta) \le 7 + 2.5
\log\left(\frac{R^2}{\lambda \gamma n}\right) ~.
\]
All in all, we have shown that the average runtime required by
Prox-SDCA$(\tilde{P}_t,\frac{\eta}{2(1+\eta^{-2})}\xi_{t-1},\alpha^{(t-1)})$ is upper bounded by
\[
O\left(d\,n \log\left(\frac{R^2}{\lambda \gamma n}\right)\right) ~,
\]
which concludes the proof of the second claim of \thmref{thm:acceleratedThmMain}.

\section{Applications} \label{sec:applications}

In this section we specify our algorithmic framework to several
popular machine learning applications. In \secref{sec:appLossFunc} we
start by describing several loss functions and deriving their
conjugate. In \secref{sec:appRegularizers} we describe several
regularization functions. Finally, in the rest of the subsections we
specify our algorithm for Ridge regression, SVM, Lasso, logistic
regression, and multiclass prediction. 

\subsection{Loss functions} \label{sec:appLossFunc}

\paragraph{Squared loss:}
$\phi(a) = \frac{1}{2} (a-y)^2$ for some $y \in \reals$. The conjugate
function is
\[
\phi^*(b) = \max_a a b - \frac{1}{2} (a-y)^2 = \frac{1}{2} b^2 + yb
\]

\paragraph{Logistic loss:}
$\phi(a) = \log(1+e^{a})$. The derivative is $\phi'(a) = 1/(1+e^{-a})$
and the second derivative is $\phi''(a) = \frac{1}{(1+e^{-a})(1+e^a)}
\in [0,1/4]$, from which it follows that $\phi$ is $(1/4)$-smooth. 
The conjugate function is
\[
\phi^*(b) = \max_a  a b - \log(1+e^{a})  = 
\begin{cases}
b \log(b) +
(1-b)\log(1-b)   & \textrm{if}~ b \in [0,1] \\
\infty & \textrm{otherwise}
\end{cases}
\]

\paragraph{Hinge loss:}
$\phi(a) = [1-a]_+ := \max\{0,1-a\}$. The conjugate function is 
\[
\phi^*(b) = \max_a a b - \max\{0,1-a\} = 
\begin{cases}
b  & \textrm{if}~ b \in [-1,0] \\
\infty & \textrm{otherwise}
\end{cases}
\]

\paragraph{Smooth hinge loss:}
This loss is obtained by smoothing the hinge-loss using the technique
described in \lemref{lem:smoothingLemma}. This loss is parameterized
by a scalar $\gamma > 0$ and is defined as:
\begin{equation} \label{eqn:smoothhinge}
\tilde{\phi}_\gamma(a) = \begin{cases}
0  & a \ge 1 \\
1-a-\gamma/2  & a \le 1-\gamma \\
\frac{1}{2\gamma}(1-a)^2    & \textrm{o.w.}
\end{cases} 
\end{equation}
The conjugate function is
\[
\tilde{\phi}_\gamma^*(b) = 
\begin{cases}
b + \frac{\gamma}{2} b^2 & \textrm{if}~ b \in [-1,0] \\
\infty & \textrm{otherwise}
\end{cases}
\]
It follows that $\tilde{\phi}_\gamma^*$ is $\gamma$ strongly convex
and $\tilde{\phi}$ is $(1/\gamma)$-smooth. In addition, if $\phi$ is
the vanilla hinge-loss, we have for every
$a$ that
\[
\phi(a)-\gamma/2 \le \tilde{\phi}(a) \le \phi(a) ~.
\]

\paragraph{Max-of-hinge:} 
The max-of-hinge loss function is a function from $\reals^{k}$ to
$\reals$, which is defined as:
\[
\phi(a) = \max_j \, [c_j + a_j]_+ ~,
\]
for some $c \in \reals^k$.  This loss function is
useful for multiclass prediction problems. 

To calculate the conjugate of $\phi$, let 
\begin{equation} \label{eqn:Sdef}
S = \{\beta \in \reals_+^k
: \|\beta\|_1 \le 1\}
\end{equation}
and note that we can write $\phi$ as
\[
\phi(a) = \max_{\beta \in S} \sum_j
\beta_j (c_j + a_j) ~.
\]
Hence, the conjugate of $\phi$ is
\begin{align*}
\phi^*(b) &= \max_{a} \left[ a^\top b - \phi(a) \right] 
= \max_{a} \min_{\beta \in S} \left[ a^\top b - \sum_j \beta_j
(c_j +  a_j) \right] \\
&= \min_{\beta \in S} \max_{a} \left[ a^\top b - \sum_j \beta_j
(c_j +  a_j)  \right] 
= \min_{\beta \in S} \left[ - \sum_j \beta_j c_j + \sum_j \max_{a_j}  a_j (b_j -
   \beta_j)\right] .
\end{align*}
Each inner maximization over $a_j$ would be $\infty$
unless $\beta_j = b_j$. Therefore,
\begin{equation} \label{eqn:maxOfHingeConj}
\phi^*(b) = \begin{cases}
- c^\top b & ~\textrm{if}~  b \in S\\
\infty & ~\textrm{otherwise}
\end{cases}
\end{equation}

\paragraph{Smooth max-of-hinge}

This loss obtained by smoothing the max-of-hinge loss using the
technique described in \lemref{lem:smoothingLemma}. This loss is parameterized by a scalar
$\gamma > 0$. We start by adding regularization to the 
conjugate of the max-of-hinge given in \eqref{eqn:maxOfHingeConj} and
obtain 
\begin{equation} \label{eqn:SmoothMaxOfHingeConj}
\tilde{\phi}^*_\gamma(b) = \begin{cases}
\frac{\gamma}{2} \|b\|^2 - c^\top b & ~\textrm{if}~  b \in S\\
\infty & ~\textrm{otherwise}
\end{cases}
\end{equation}

Taking the conjugate of the conjugate we obtain
\begin{align} \nonumber
\tilde{\phi}_\gamma(a) &= \max_{b} b^\top a -
\tilde{\phi}^*_\gamma(b) \\ \nonumber
&= \max_{b \in S} b^\top (a+c) - \frac{\gamma}{2} \|b\|^2 \\
&= \frac{\gamma}{2} \|(a+c)/\gamma\|^2 - \frac{\gamma}{2} \min_{b \in S} \|b -
(a+c)/\gamma\|^2 \label{eqn:SmoothMaxOfHingeConj}
\end{align}
While we do not have a closed form solution for the minimization
problem over $b$ in the definition of $\tilde{\phi}_\gamma$ above,
this is a problem of projecting onto the intersection of the $L_1$
ball and the positive orthant, and can be solved efficiently using the
following procedure, adapted from \cite{duchi2008efficient}.

\begin{myalgo}{Project$(\mu)$}
\textbf{Goal:} solve $\argmin_b \|b-\mu\|^2 ~\textrm{s.t.}~ b \in
\reals_+^k, \|b\|_1 \le 1$ \\
\textbf{Let:} $\forall i, ~\tilde{\mu}_i = \max\{0,\mu_i\}$ \\
\textbf{If:} $\|\tilde{\mu}\|_1 \le 1$ stop and return $b =
\tilde{\mu}$ \\
\textbf{Sort:} let $i_1,\ldots,i_k$ be s.t. $\mu_{i_1} \ge \mu_{i_2}
\ge \ldots \ge \mu_{i_k}$ \\
\textbf{Find:} $j^* = \max\left\{ j : j\,\tilde{\mu}_{i_j} + 1 -  
  \sum_{r=1}^j \tilde{\mu}_{i_r} > 0 \right\}$ \\
\textbf{Define:} $\theta = -1 + \sum_{r=1}^{j^*} \tilde{\mu}_{i_r} $ \\
\textbf{Return:} $b$ s.t. $\forall i,~b_i = \max\{\mu_i - \theta/j^*, 0 \}$ 
\end{myalgo}

It also holds that $\nabla \tilde{\phi}_\gamma(a) = \argmin_{b \in S}
\|b - (a+c)/\gamma\|^2$, and therefore the gradient can also be
calculated using the above projection procedure. 

Note that if $\phi$ being the max-of-hinge loss, then 
$\phi^*(b)+\gamma/2 \ge \tilde{\phi}^*_\gamma(b) \ge
\phi^*(b)$ and hence $\phi(a) - \gamma/2 \le \tilde{\phi}_\gamma(a)
\le \phi(a)$. 

Observe that all negative elements of $a+c$ does not contribute to
$\tilde{\phi}_\gamma$. This immediately implies  that if $\phi(a) = 0$
then we also have $\tilde{\phi}_\gamma(a)=0$. 


\paragraph{Soft-max-of-hinge loss function:}
Another approach to smooth the max-of-hinge loss function is by using
soft-max instead of max. The resulting soft-max-of-hinge loss function
is defined as
\begin{equation} \label{eqn:soft-max-loss}
\phi_\gamma(a) = \gamma \log\left( 1 + \sum_{i=1}^k e^{(c_i+a_i)/\gamma}  \right) ~,
\end{equation}
where $\gamma > 0$ is a parameter.
We have
\[
\max_i [c_i+a_i]_+ \le \phi_\gamma(a) \le \max_i [c_i+a_i]_+ +
\gamma\,\log(k+1)~.
\]

The $j$'th element of the gradient of $\phi$ is
\[
\nabla_j \phi_\gamma(a) = \frac{e^{(c_j+a_j)/\gamma}}{1 + \sum_{i=1}^k
  e^{(c_i+a_i)/\gamma} } ~.
\]
By the definition of the conjugate we have $\phi_\gamma^*(b) = \max_a
a^\top b - \phi_\gamma(a)$.  The vector $a$ that maximizes the above
must satisfy
\[
\forall j,~~ b_j =  \frac{e^{(c_j+a_j)/\gamma}}{1 + \sum_{i=1}^k
  e^{(c_i+a_i)/\gamma} } ~.
\]
This can be satisfied only if $b_j \ge 0$ for all $j$ and $\sum_j b_j
\le 1$. That is, $b \in S$. Denote $Z = \sum_{i=1}^k
  e^{(c_i+a_i)/\gamma}$ and note that
\[
(1+Z) \|b\|_1 = Z ~~\Rightarrow~~ Z = \frac{\|b\|_1}{1-\|b\|_1} 
~~\Rightarrow~~ 1+Z = \frac{1}{1-\|b\|_1}  ~.
\]
It follows that 
\[a_j = \gamma(\log(b_j) +
    \log(1+Z)) - c_j  = \gamma(\log(b_j) -
    \log(1-\|b\|_1)) - c_j
\]
which yields
\begin{align*}
\phi_\gamma^*(b)  &= \sum_j \left(\gamma(\log(b_j) -
    \log(1-\|b\|_1)) - c_j\right) b_j + \gamma \log( 1 - \|b\|_1) \\
&= -c^\top b + \gamma \left((1-\|b\|_1)\log(1-\|b\|_1) + \sum_j b_j \log(b_j) 
  \right) ~.
\end{align*}
Finally, if $b \notin S$ then the gradient of $a^\top b -
\phi_\gamma(a)$ does not vanish anywhere, which means that
$\phi_\gamma^*(b) = \infty$. All in all, we obtain
\begin{equation} \label{eqn:soft-max-loss-conjugate}
\phi_\gamma^*(b) = 
\begin{cases}
-c^\top b + \gamma \left((1-\|b\|_1)\log(1-\|b\|_1) + \sum_j b_j \log(b_j) 
  \right)  & ~\textrm{if}~ b \in S \\
\infty & ~\textrm{otherwise}
\end{cases}
\end{equation}
Since the entropic function, $\sum_j b_j \log(b_j)$ is $1$-strongly
convex over $S$ with respect to the $L_1$ norm, we obtain that
$\phi^*_\gamma$ is $\gamma$-strongly convex with respect to the $L_1$
norm, from which it follows that $\phi_\gamma$ is $(1/\gamma)$-smooth
with respect to the $L_\infty$ norm.

\subsection{Regularizers} \label{sec:appRegularizers}

\paragraph{$L_2$ regularization:}
The simplest regularization is the squared $L_2$ regularization
\[
g(w) = \frac{1}{2} \|w\|_2^2 ~.
\]
This is a $1$-strongly convex regularization function whose conjugate is 
\[
g^*(\theta) = \frac{1}{2} \|\theta\|_2^2 ~.
\]
We also have
\[
\nabla g^*(\theta) = \theta ~.
\]

For our acceleration procedure, we also use the $L_2$ regularization
plus a linear term, namely, 
\[
g(w) = \frac{1}{2} \|w\|^2  -  w^\top z ~,
\]
for some vector $z$. The conjugate of this function is
\[
g^*(\theta) = \max_w \left[w^\top (\theta+z) - \frac{1}{2} \|w\|^2
\right] = \frac{1}{2} \|\theta+z\|^2 ~.
\]
We also have
\[
\nabla g^*(\theta) = \theta + z~.
\]

\paragraph{$L_1$ regularization:}
Another popular regularization we consider is the $L_1$
regularization, 
\[
f(w) =  \sigma\, \|w\|_1 ~.
\]
This is not a strongly convex regularizer and therefore we will add a
slight $L_2$ regularization to it and define the $L_1$-$L_2$
regularization as
\begin{equation} \label{eqn:gdefl1l2}
g(w) = \frac{1}{2} \|w\|_2^2 + \sigma'\, \|w\|_1 ~,
\end{equation}
where $\sigma' = \frac{\sigma}{\lambda}$ for some small $\lambda$. Note that 
\[
\lambda g(w) = \frac{\lambda}{2} \|w\|_2^2 + \sigma \|w\|_1 ~,
\]
so if $\lambda$ is small enough (as will be formalized later) we
obtain that $\lambda g(w) \approx \sigma \|w\|_1$.

The conjugate of $g$ is
\begin{align*}
g^*(v) &= \max_{w} \left[w^\top v - \frac{1}{2} \|w\|_2^2 - \sigma'
  \|w\|_1 \right] ~.
\end{align*}
The maximizer is also $\nabla g^*(v)$ and we now show how to calculate it. We have
\begin{align*}
\nabla g^*(v) &= \argmax_{w} \left[w^\top v - \frac{1}{2} \|w\|_2^2 -
\sigma' \|w\|_1 \right] \\
&= \argmin_w \left[ \frac{1}{2} \|w-v\|_2^2 +
  \sigma' \|w\|_1 \right]
\end{align*}
A sub-gradient of the objective of the optimization problem above is
of the form $w-v +
\sigma' z = 0$, where $z$ is a vector with $z_i =
\sign(w_i)$, where if $w_i=0$ then $z_i \in [-1,1]$. Therefore, if $w$
is an optimal solution then for all $i$, either $w_i=0$ or $w_i = v_i
- \sigma' \sign(w_i)$. Furthermore, it is easy to
verify that if $w$ is an optimal solution then for all $i$, if $w_i
\neq 0$ then the sign
of $w_i$ must be the sign of $v_i$. Therefore, whenever $w_i \neq 0$
we have that  $w_i = v_i - \sigma' \sign(v_i)$. It
follows that in that case we must have $|v_i| >
\sigma'$. And, the other direction is also true,
namely, if $|v_i| > \sigma'$ then setting $w_i = v_i -
\sigma' \sign(v_i)$ leads to an objective value whose
$i$'th component is
\[
\frac{1}{2} \left(\sigma'\right)^2 + \sigma' (|v_i| -
\sigma') \le \frac{1}{2} |v_i|^2 ~,
\]
where the right-hand side is the $i$'th component of the objective value we will obtain by
setting $w_i=0$. This leads to the conclusion that 
\[
\nabla_i g^*(v) = \sign(v_i)\left[ |v_i| - \sigma'\right]_+ = \begin{cases}
v_i - \sigma' \sign(v_i) & \textrm{if}~ |v_i| >
\sigma' \\
0 & \textrm{o.w.}
\end{cases}
\]
It follows that 
\begin{align*}
g^*(v) &=  \sum_i \sign(v_i)\left[ |v_i| -
  \sigma'\right]_+  \, v_i - \frac{1}{2} \sum_i (\left[ |v_i| - \sigma'\right]_+)^2 - \sigma'
  \sum_i \left[ |v_i| - \sigma'\right]_+ \\
&= \sum_i \left[ |v_i| -
  \sigma'\right]_+ \left( |v_i| - \sigma' - \frac{1}{2}\left[ |v_i| -
    \sigma'\right]_+  \right) \\
&= \frac{1}{2} \sum_i \left(\left[ |v_i| - \sigma'\right]_+\right)^2 ~.
\end{align*}

Another regularization function we'll use in the accelerated procedure
is
\begin{equation} \label{eqn:gdefl1l2acc}
g(w) = \frac{1}{2} \|w\|_2^2 + \sigma'\, \|w\|_1 - z^\top w ~.
\end{equation}
The conjugate function is
\[
g^*(v) = \frac{1}{2} \sum_i \left(\left[ |v_i + z_i| - \sigma'\right]_+\right)^2 ~,
\]
and its gradient is
\[
\nabla_i g^*(v) = \sign(v_i + z_i)\left[ |v_i + z_i| - \sigma'\right]_+ 
\]

\subsection{Ridge Regression}

In ridge regression, we minimize the squared loss with $L_2$
regularization. That is, $g(w) = \frac{1}{2} \|w\|^2$ and for every
$i$ we have that $x_i \in \reals^d$ and $\phi_i(a) = \frac{1}{2}
(a-y_i)^2$ for some $y_i \in \reals$. The primal problem is therefore
\[
P(w) = \frac{1}{2n} \sum_{i=1}^n (x_i^\top w - y_i)^2  +
\frac{\lambda}{2} \|w\|^2 ~.
\]

Below we specify Prox-SDCA for ridge regression. We use Option I since
it is possible to derive a closed form solution to the maximization of
the dual with respect to $\Delta \alpha_i$. Indeed, since 
$-\phi_i^*(-b) = -\frac{1}{2} b^2 + y_i b $ we have that 
the maximization problem is
\begin{align*}
\Delta \alpha_i &=~ \argmax_{b} - \frac{1}{2} (\alpha^{(t+1)}_i + b)^2 + y_i
(\alpha^{(t+1)}_i + b) - w^{(t-1) \top} x_i b - \frac{b^2
  \|x_i\|^2}{2\lambda n} \\
&=~\argmax_{b} - \frac{1}{a} \left(1 + \frac{\|x_i\|^2}{2\lambda n} \right) b^2 
- \left( \alpha^{(t+1)}_i + w^{(t-1) \top} x_i - y_i\right) \,b\\
&=~ - \frac{\alpha^{(t+1)}_i  + w^{(t-1) \top} x_i - y_i}{1 +
  \frac{\|x_i\|^2}{2\lambda n} } ~.
\end{align*}

Applying the above update and using some additional tricks to improve
the running time we obtain the following procedure. 
\begin{myalgo}{Prox-SDCA($(x_i,y_i)_{i=1}^n,\epsilon,\alpha^{(0)},z$)
   for solving ridge regression} 
\textbf{Goal:} Minimize $P(w) = \frac{1}{2n} \sum_{i=1}^n
(x_i^\top w -y_i)^2 + \lambda \left(\frac{1}{2} \|w\|^2 -  w^\top z\right)$ \\
\textbf{Initialize} $v^{(0)}=\frac{1}{\lambda n} \sum_{i=1}^n 
\alpha_i^{(0)} x_i$, $\forall i,~ \tilde{y}_i = y_i - x_i^\top z$ \\
\textbf{Iterate:} for $t=1,2,\dots$ \+ \\
 Randomly pick $i$ \\
 $\Delta \alpha_i = - \frac{\alpha^{(t-1)}_i  + v^{(t-1) \top} x_i - \tilde{y}_i}{1 +
  \frac{\|x_i\|^2}{2\lambda n} } $ \\
  $\alpha^{(t)}_i \leftarrow \alpha^{(t-1)}_i + \Delta \alpha_i$ and
  for $j \neq i$, $\alpha^{(t)}_j \leftarrow \alpha^{(t-1)}_j$ \\
  $v^{(t)} \leftarrow v^{(t-1)} + \frac{\Delta \alpha_i}{\lambda n}x_i $ \- \\
  \textbf{Stopping condition}: \+ \\
Let $w^{(t)} = v^{(t)} + z$ \\
  Stop if $\frac{1}{2n} \sum_{i=1}^n \left( ( x_i^\top w^{(t)}-y_i)^2
    + (\alpha^{(t)}_i + y_i)^2 - y_i^2\right)  + \lambda w^{(t)\,^\top}
  v^{(t)} \le \epsilon$ 
  
\end{myalgo}

The runtime of Prox-SDCA for ridge regression becomes 
\[
\tilde{O}\left(d\left(n+
      \frac{R^2}{\lambda}\right)\right) ~,
\]
where $R = \max_i \|x_i\|$. This matches the recent results of
\cite{LSB12-sgdexp,ShalevZh2013}. If $R^2/\lambda \gg n$ we can apply
the accelerated procedure and obtain the improved runtime
\[
\tilde{O}\left(d\sqrt{\frac{nR^2}{\lambda}}\right) ~.
\]

\subsection{Logistic Regression}

In logistic regression, we minimize the logistic loss with $L_2$
regularization. That is, $g(w) = \frac{1}{2} \|w\|^2$ and for every
$i$ we have that $x_i \in \reals^d$ and $\phi_i(a) = log(1+e^a)$. The
primal problem is therefore\footnote{Usually, the training data comes
  with labels, $y_i \in \{\pm 1\}$,
  and the loss function becomes $\log(1+e^{-y_i x_i^\top
    w})$. However, we can easily get rid of the labels by re-defining $x_i \leftarrow -y_i x_i$.}
\[
P(w) = \frac{1}{n} \sum_{i=1}^n \log(1+e^{x_i^\top w})  +
\frac{\lambda}{2} \|w\|^2 ~.
\]
The dual problem is
\[
D(\alpha) = \frac{1}{n} \sum_{i=1}^n (\alpha_i\log(-\alpha_i) - (1+\alpha_i)\log(1+\alpha_i))  -
\frac{\lambda}{2} \|v(\alpha)\|^2 ~,
\]
and the dual constraints are $\alpha \in [-1,0]^n$. 

Below we specify Prox-SDCA for logistic regression using Option III.
\begin{myalgo}{Prox-SDCA($(x_i)_{i=1}^n,\epsilon,\alpha^{(0)},z$)
   for logistic regression} 
\textbf{Goal:} Minimize $P(w) = \frac{1}{n} \sum_{i=1}^n
\log(1 + e^{x_i^\top w}) + \lambda \left(\frac{1}{2} \|w\|^2 -  w^\top z\right)$ \\
\textbf{Initialize} $v^{(0)}=\frac{1}{\lambda n} \sum_{i=1}^n 
\alpha_i^{(0)} x_i$, and $\forall i,~~ p_i = x_i^\top z$ \\
\textbf{Define:} $\phi^*(b) = b\log(b)+(1-b)\log(1-b)$ \\
\textbf{Iterate:} for $t=1,2,\dots$ \+ \\
 Randomly pick $i$ \\
 $p = x_i^\top w^{(t-1)}$ \\
 $q = -1/(1+e^{-p}) -\alpha_i^{(t-1)}$ \\
 $s = \min\left(1,\frac{\log(1+e^p) + \phi^*(-\alpha_i^{(t-1)}) + p \alpha^{(t-1)}_i
  + 2 q^2}{ q^2 (4 + \frac{1}{\lambda n}
  \|x_i\|^2 )}\right)$ \\
 $\Delta \alpha_i = sq $ \\
  $\alpha^{(t)}_i = \alpha^{(t-1)}_i + \Delta \alpha_i$ and
  for $j \neq i$, $\alpha^{(t)}_j = \alpha^{(t-1)}_j$ \\
  $v^{(t)} = v^{(t-1)} + \frac{\Delta \alpha_i}{\lambda n}x_i
  $ \- \\
  \textbf{Stopping condition}: \+ \\
 let $w^{(t)} = v^{(t)} + z$ \\
  Stop if $\frac{1}{n} \sum_{i=1}^n \left( \log(1 + e^{x_i^\top
    w^{(t)}}) + \phi^*(-\alpha_i^{(t-1)}) \right) +
  \lambda w^{(t) \top} v^{(t)} \le \epsilon$ 
\end{myalgo}

The runtime analysis is similar to the analysis for ridge regression.

\subsection{Lasso}

In the Lasso problem, the loss function is the squared loss but the
regularization function is $L_1$. That is, we need to solve the problem:
\begin{equation} \label{eqn:Lasso}
\min_w \left[ \frac{1}{2n} \sum_{i=1}^n ( x_i^\top w - y_i)^2 + \sigma
  \|w\|_1 \right] ~,
\end{equation}
with a positive regularization parameter $\sigma \in \reals_+$.

Let $\bar{y} = \frac{1}{2n}
\sum_{i=1}^n y_i^2$, and let $\bar{w}$ be an optimal solution of
\eqref{eqn:Lasso}. Then, the objective at $\bar{w}$ is at most the
objective at $w=0$, which yields
\[
\sigma \|\bar{w}\|_1 \le \bar{y} ~~\Rightarrow~~ \|\bar{w}\|_2 \le \|\bar{w}\|_1
\le \frac{\bar{y}}{\sigma} ~.
\]
Consider the optimization problem
\begin{equation} \label{eqn:LassoL1L2}
\min_w P(w) ~~~\textrm{where}~~~ P(w)  = \frac{1}{2n} \sum_{i=1}^n ( x_i^\top w - y_i)^2 +
  \lambda \left( \half \|w\|_2^2 + \frac{\sigma}{\lambda}
  \|w\|_1 \right)  ~,
\end{equation}
for some $\lambda > 0$. This problem fits into our framework, since
now the regularizer is strongly convex.  Furthermore, if $w^*$
is an $(\epsilon/2)$-accurate solution to the problem in
\eqref{eqn:LassoL1L2},  then $P(w^*) \le P(\bar{w}) +\epsilon/2$ which yields
\[
\left[\frac{1}{2n} \sum_{i=1}^n ( x_i^\top w^* - y_i)^2 + \sigma \|w^*\|_1
\right] \le \left[
\frac{1}{2n} \sum_{i=1}^n ( x_i^\top \bar{w} - y_i)^2 + \sigma
\|\bar{w}\|_1 \right] + \frac{\lambda}{2} \|\bar{w}\|_2^2 + \epsilon/2~.
\]
Since $\|\bar{w}\|_2^2 \le \left({\bar{y}}/{\sigma} \right)^2$, we
obtain that setting $\lambda = \epsilon (\sigma/\bar{y})^2$ guarantees
that $w^*$ is an $\epsilon$ accurate solution to the original problem
given in \eqref{eqn:Lasso}.

In light of the above, from now on we focus on the problem given in
\eqref{eqn:LassoL1L2}. As in the case of ridge regression, we can
apply Prox-SDCA with Option I. The resulting pseudo-code is given
below.
Applying the above update and using some additional tricks to improve
the running time we obtain the following procedure. 
\begin{myalgo}{Prox-SDCA($(x_i,y_i)_{i=1}^n,\epsilon,\alpha^{(0)},z$)
   for solving $L_1-L_2$ regression} 
\textbf{Goal:} Minimize $P(w) = \frac{1}{2n} \sum_{i=1}^n
(x_i^\top w -y_i)^2 + \lambda \left(\frac{1}{2} \|w\|^2 +
 \sigma' \|w\|_1 -  w^\top z\right)$ \\
\textbf{Initialize} $v^{(0)}=\frac{1}{\lambda n} \sum_{i=1}^n 
\alpha_i^{(0)} x_i$, and $\forall j,~w^{(0)}_j = \sign(v_j^{(0)}+z_j)[
|v_j^{(0)}+z_j| - \sigma']_+$ \\
\textbf{Iterate:} for $t=1,2,\dots$ \+ \\
 Randomly pick $i$ \\
 $\Delta \alpha_i = - \frac{\alpha^{(t-1)}_i  + w^{(t-1) \top} x_i - y_i}{1 +
  \frac{\|x_i\|^2}{2\lambda n} } $ \\
  $\alpha^{(t)}_i = \alpha^{(t-1)}_i + \Delta \alpha_i$ and
  for $j \neq i$, $\alpha^{(t)}_j = \alpha^{(t-1)}_j$ \\
  $v^{(t)} = v^{(t-1)} + \frac{\Delta \alpha_i}{\lambda n}x_i
  $ \\
   $\forall j,~w^{(t)}_j = \sign(v_j^{(t)}+z_j)[
|v_j^{(t)}+z_j| - \sigma']_+$\- \\
  \textbf{Stopping condition}: \+ \\
  Stop if $\frac{1}{2n} \sum_{i=1}^n \left( (x_i^\top
    w^{(t)}-y_i)^2 - 2y_i\alpha^{(t)}_i + (\alpha^{(t)}_i)^2 \right) +
  \lambda w^{(t) \top} v^{(t)} \le \epsilon$ 
\end{myalgo}

Let us now discuss the runtime of the resulting method. 
Denote $R=\max_i \|x_i\|$ and for
simplicity, assume that $\bar{y} =
O(1)$. Choosing $\lambda =  \epsilon (\sigma/\bar{y})^2$, the runtime of our method becomes
\[
\tilde{O}\left(d\left(n+
  \min\left\{\frac{R^2}{\epsilon\,\sigma^2},\sqrt{\frac{nR^2}{\epsilon\,\sigma^2}}\right\}\right)\right)
~.
\]
It is also convenient to write the bound in terms of $B =
\|\bar{w}\|_2$, where, as before, $\bar{w}$ is the optimal solution of
the $L_1$ regularized problem. With this parameterization, we can set
$\lambda = \epsilon/B^2$ and the runtime becomes
\[
\tilde{O}\left(d\left(n+
  \min\left\{\frac{R^2B^2}{\epsilon},\sqrt{\frac{n\,R^2B^2}{\epsilon}}\right\}\right)\right)
~.
\]


The runtime of standard SGD is $O(dR^2 B^2 / \epsilon^2)$ even in the
case of smooth loss functions such as the squared loss. Several
variants of SGD, that leads to sparser intermediate solutions, have
been proposed
(e.g. \cite{LangfordLiZh09,shalev2011stochastic,Xiao10,duchi2009efficient,DuchiShSiTe10}). However,
all of these variants share the runtime of $O(dR^2 B^2 / \epsilon^2)$,
which is much slower than our runtime when $\epsilon$ is small.

Another relevant approach is the FISTA algorithm of 
\cite{beck2009fast}. The shrinkage operator of FISTA is the same as
the gradient of $g^*$ used in our approach.  
It is a batch algorithm using Nesterov's accelerated gradient technique.
For the squared loss
function, the runtime of FISTA is
\[
O\left( d\,n\, \sqrt{\frac{R^2B^2}{\epsilon }} \right)  ~.
\]
This bound is worst than our bound by a factor of at least $\sqrt{n}$. 

Another approach to solving \eqref{eqn:Lasso} is stochastic coordinate
descent over the primal problem. \cite{shalev2011stochastic} showed
that the runtime of this approach is
\[
O\left(\frac{dnB^2}{\epsilon}\right) ~,
\]
under the assumption that $\|x_i\|_\infty \le 1$ for all $i$. Similar results can also be found in \cite{Nesterov10}.

For our method, the runtime depends on $R^2 = \max_i \|x_i\|_2^2$. If
$R^2 = O(1)$ then the runtime of our method is much better than that
of \cite{shalev2011stochastic}. In the general case, if  $\max_i
\|x_i\|_\infty \le 1$ then $R^2 \le d$, which yields the runtime of
\[
\tilde{O}\left(d\left(n+
  \min\left\{\frac{dB^2}{\epsilon},\sqrt{\frac{n\,dB^2}{\epsilon}}\right\}\right)\right)
~.
\]
This is the same or better than \cite{shalev2011stochastic} whenever
$d = O(n)$. 

\subsection{Linear SVM}

Support Vector Machines (SVM) is an algorithm for learning a linear
classifier. Linear SVM (i.e., SVM with linear kernels) amounts to
minimizing the objective
\[
P(w) = \frac{1}{n} \sum_{i=1}^n [1 - x_i^\top w]_+ + \frac{\lambda}{2}
\|w\|^2 ~,
\]
where $[a]_+ = \max\{0,a\}$, and for every $i$, $x_i \in \reals^d$. 
This can be cast as the objective given in \eqref{eqn:PrimalProblem}
by letting the regularization be $g(w) = \frac{1}{2} \|w\|_2^2$,
and for every $i$, $\phi_i(a) = [1-a]_+$, is the hinge-loss. 

Let $R =\max_i \|x_i\|_2$. SGD enjoys the rate of
$O\left(\frac{1}{\lambda \epsilon}\right)$. Many software packages
apply SDCA and obtain the rate $\tilde{O}\left(n + \frac{1}{\lambda
    \epsilon}\right)$. We now show how our accelerated proximal SDCA
enjoys the rate $\tilde{O}\left(n + \sqrt{\frac{n}{\lambda
      \epsilon}}\right)$. This is significantly better than the rate
of SGD when $\lambda \epsilon < 1/n$. We note that a default setting
for $\lambda$, which often works well in practice, is $\lambda =
1/n$. In this case, $\lambda \epsilon = \epsilon/n \ll 1/n$.

Our first step is to smooth the hinge-loss. Let $\gamma = \epsilon$
and consider the smooth hinge-loss as defined in
\eqref{eqn:smoothhinge}. 
Recall that the smooth hinge-loss satisfies
\[
\forall a,~~\phi(a)-\gamma/2 \le \tilde{\phi}(a) \le \phi(a) ~.
\]
Let $\tilde{P}$ be the SVM objective while replacing the hinge-loss
with the smooth hinge-loss. Therefore,  for every $w'$ and $w$, 
\[
P(w')- P(w) \le \tilde{P}(w') - \tilde{P}(w) + \gamma/2 ~.
\]
It follows that if $w'$ is an $(\epsilon/2)$-optimal solution for
$\tilde{P}$, then it is $\epsilon$-optimal solution for $P$. 

For the smoothed hinge loss, the optimization problem given in Option
I of Prox-SDCA has a closed form solution and we obtain the following
procedure:

\begin{myalgo}{Prox-SDCA($(x_1,\ldots,x_n),\epsilon,\alpha^{(0)},z$)
   for solving SVM (with smooth hinge-loss as in \eqref{eqn:smoothhinge})} 
\textbf{Define:} $\tilde{\phi}_\gamma$ as in \eqref{eqn:smoothhinge} \\
\textbf{Goal:} Minimize $P(w) = \frac{1}{n} \sum_{i=1}^n
\tilde{\phi}_\gamma(x_i^\top w) + \lambda \left(\frac{1}{2} \|w\|^2 -  w^\top z\right)$ \\
\textbf{Initialize} $w^{(0)}=z + \frac{1}{\lambda n} \sum_{i=1}^n 
\alpha_i^{(0)} x_i$ \\
\textbf{Iterate:} for $t=1,2,\dots$ \+ \\
 Randomly pick $i$ \\
$\Delta \alpha_i  =  \max\left(-\alpha^{(t-1)}_i~,~ \min\left(1 -\alpha^{(t-1)}_i~,~ \frac{1 -  x_i^\top
      w^{(t-1)}  - \gamma\,\alpha^{(t-1)}_i}{\|x_i\|^2/(\lambda n)+\gamma} \right) \right) $\\
  $\alpha^{(t)}_i \leftarrow \alpha^{(t-1)}_i + \Delta \alpha_i$ and
  for $j \neq i$, $\alpha^{(t)}_j \leftarrow \alpha^{(t-1)}_j$ \\
  $w^{(t)} \leftarrow w^{(t-1)} + \frac{\Delta \alpha_i}{\lambda n}x_i $ \- \\
  \textbf{Stopping condition}: \+ \\
  Stop if $\frac{1}{n} \sum_{i=1}^n \left(
    \tilde{\phi}_\gamma(x_i^\top w^{(t)}) - \alpha^{(t)}_i +
    \frac{\gamma}{2} (\alpha^{(t)}_i)^2 \right) + \lambda w^{(t)^\top}
  (w^{(t)}-z) \le \epsilon$
\end{myalgo}

Denote $R=\max_i \|x_i\|$. Then, the runtime of the resulting method is
\[
\tilde{O}\left(d\left(n+
  \min\left\{\frac{R^2}{\gamma\,\lambda},\sqrt{\frac{nR^2}{\gamma\,\lambda}}\right\}\right)\right)
~.
\]
In particular, choosing $\gamma = \epsilon$ we obtain a solution to
the original SVM problem in runtime of 
\[
\tilde{O}\left(d\left(n+
  \min\left\{\frac{R^2}{\epsilon\,\lambda},\sqrt{\frac{nR^2}{\epsilon\,\lambda}}\right\}\right)\right)
~.
\]
As mentioned before, this is better than SGD when $\frac{1}{\lambda
  \epsilon} \gg n$.

\subsection{Multiclass SVM}

Next we consider Multiclass SVM using the construction described in
\citet{CrammerSi01a}. Each example consists of an instance vector $x_i \in
\reals^d$ and a label $y_i \in \{1,\ldots,k\}$. The goal is to learn a
matrix $W \in \reals^{d,k}$ such that $W^\top x_i$ is a $k$'th dimensional
vector of scores for the different classes. The prediction is the
coordinate of $W^\top x_i$ of maximal value. The loss function is 
\[
\max_{j \neq y_i}  (1 + (W^\top x_i)_j - (W^\top x_i)_{y_i} ) ~.
\]
This can be written as $\phi((W^\top x_i) - (W^\top x_i)_{y_i})$ where
\[
\phi_i(a) = \max_j [c_{i,j} + a_j]_+ ~,
\] 
with $c_i$ being the all ones vector except $0$ in the $y_i$
coordinate. 

We can model this in our framework as follows. Given a matrix $M$ let
$\textrm{vec}(M)$ be the column vector obtained by concatenating the columns
of $M$. Let $e_j$ be the all zeros vector except $1$ in the $j$'th coordinate.
For every $i$, let $c_i = \mathbf{1} - e_{y_i}$ and 
let $X_i \in \reals^{dk,k}$ be the matrix whose $j$'th column
is $\textrm{vec}(x_i (e_j-e_{y_i})^\top)$. Then, 
\[
X_i^\top \textrm{vec}(W) = W^\top x_i - (W^\top x_i)_{y_i} ~.
\]
Therefore, the optimization problem of multiclass SVM becomes:
\[
\min_{w \in \reals^{dk}} P(w) ~~~\textrm{where}~~~
P(w) = \frac{1}{n} \sum_{i=1}^n \phi_i(X_i^\top w)  +
\frac{\lambda}{2} \|w\|^2 ~.
\]

As in the case of SVM, we will use the smooth version of the
max-of-hinge loss function as described in \eqref{eqn:SmoothMaxOfHingeConj}.
If we set the smoothness parameter $\gamma$ to be $\epsilon$ then an
$(\epsilon/2)$-accurate solution to the problem with the smooth loss
is also an $\epsilon$-accurate solution to the original problem with
the non-smooth loss. Therefore, from now on we focus on the problem
with the smooth max-of-hinge loss.

We specify Prox-SDCA for multiclass SVM using Option I. We will show
that the optimization problem in Option I can be calculated
efficiently by sorting a $k$ dimensional vector. Such ideas were
explored in \cite{CrammerSi01a} for the non-smooth max-of-hinge loss. 

Let $\hat{w} = w - \frac{1}{\lambda n} X_i \alpha^{(t-1)}_i$. Then, the optimization
problem over $\alpha_i$ can be written as
\begin{equation} \label{eqn:OptDualForMultiVector}
\argmax_{\alpha_i : -\alpha_i \in S} ~~(-c_i^\top - \hat{w}^\top X_i) \alpha_i
  - \frac{\gamma}{2} \|\alpha_i\|^2  - \frac{1}{2\lambda n} \|X_i \alpha_i \|^2 ~.
\end{equation}
As shown before, if we organize $\hat{w}$ as a $d \times k$ matrix, denoted $\hat{W}$,
we have that $X_i^\top \hat{w} = \hat{W}^\top x_i - (\hat{W}^\top x_i)_{y_i}$.  We also have that
\[
X_i \alpha_i = \sum_j \textrm{vec}(x_i (e_j-e_{y_i})^\top)
\alpha_{i,j} =  \textrm{vec}(x_i \sum_j \alpha_{i,j} (e_j - e_{y_i})^\top)
= \textrm{vec}(x_i (\alpha_i  - \|\alpha_i\|_1 e_{y_i})^\top) ~.
\]
It follows that an optimal solution to
\eqref{eqn:OptDualForMultiVector} must set $\alpha_{i,y_i} = 0$ and we
only need to optimize over the rest of the dual variables. This also
yields,
\[
\|X_i\alpha_i\|^2 = \|x_i\|^2 \|\alpha_i\|_2^2 +
\|x_i\|^2 \|\alpha_i\|_1^2 ~.
\]
So, \eqref{eqn:OptDualForMultiVector} becomes: 
\begin{equation} \label{eqn:OptDualForMultiVector2}
\argmax_{\alpha_i : -\alpha_i \in S, \alpha_{i,y_i}=0} ~~(-c_i^\top - \hat{w}^\top X_i) \alpha_i
  - \frac{\gamma}{2} \|\alpha_i\|_2^2  - \frac{\|x_i\|^2}{2\lambda n}
  \|\alpha_i \|_2^2 
- \frac{\|x_i\|^2}{2\lambda n} \|\alpha_i \|_1^2 ~.
\end{equation}
This is equivalent to a problem of the form:
\begin{equation} \label{eqn:aProxForMC}
\argmin_{a \in \reals_+^{k-1}, \beta}   \|a - \mu\|_2^2 +
C \beta^2 ~~~\textrm{s.t.}~~~  \|a\|_1 = \beta \le 1 ~,
\end{equation}
where
\[
\mu = \frac{c_i^\top + \hat{w}^\top X_i}{\gamma + \frac{\|x_i\|^2}{\lambda
    n}} ~~~\textrm{and}~~~ C = \frac{\frac{\|x_i\|^2}{\lambda
    n}}{\gamma + \frac{\|x_i\|^2}{\lambda
    n}} = \frac{1}{ \frac{\gamma \lambda n}{\|x_i\|^2} + 1} ~.
\]
The equivalence is in the sense that if $(a,\beta)$ is a solution of
\eqref{eqn:aProxForMC} then we can set $\alpha_i = -a$.

Assume for simplicity that $\mu$ is sorted in a non-increasing order
and that all of its elements are non-negative (otherwise, it is easy
to verify that we can zero the negative elements of $\mu$ and sort the
non-negative, without affecting the solution). Let $\bar{\mu}$ be the
cumulative sum of $\mu$, that is, for every $j$, let $\bar{\mu}_j =
\sum_{r=1}^j \mu_r$.  For every $j$, let $z_j = \bar{\mu}_j - j
\mu_j$. Since $\mu$ is sorted we have that
\[
z_{j+1} = \sum_{r=1}^{j+1} \mu_r - (j+1) \mu_{j+1} = 
\sum_{r=1}^j \mu_r
- j \mu_{j+1} \ge \sum_{r=1}^j \mu_r
- j \mu_{j} = z_j ~.
\]
Note also that $z_1 = 0$ and that $z_{k} = \bar{\mu}_k = \|\mu\|_1$
(since the coordinate of $\mu$ that corresponds to $y_i$ is zero).  By
the properties of projection onto the simplex (see
\cite{duchi2008efficient}), for every $z \in (z_j,z_{j+1})$ we have
that the projection of $\mu$ onto the set $\{b \in \reals_+^k :
\|b\|_1=z\}$ is of the form $a_r = \max\{0,\mu_r - \theta/j\}$ where
$\theta = (-z + \bar{\mu}_j)/j$. Therefore, the objective becomes
(ignoring constants that do not depend on $z$),
\[
j \theta^2 + C z^2 = (-z + \bar{\mu}_j)^2/j + C z^2 ~.
\]
The first order condition for minimality w.r.t. $z$ is
\[
-(-z + \bar{\mu}_j)/j + Cz = 0 ~~\Rightarrow~~
z = \frac{\bar{\mu}_j}{1 + jC} ~.
\]
If this value of $z$ is in $(z_j,z_{j+1})$, then it is the optimal $z$
and we're done. Otherwise, the optimum should be either $z=0$ (which
yields $\alpha=0$ as well) or $z=1$. 

\begin{myalgo}{$a = \textrm{OptimizeDual}(\mu,C)$ }
\textbf{Solve} the optimization problem given in
\eqref{eqn:aProxForMC} \\
\textbf{Initialize:} $\forall i, ~ \hat{\mu}_i = \max\{0,\mu_i\}$,
and sort $\hat{\mu}$ s.t. $\hat{\mu}_1 \ge \hat{\mu}_2 \ge \ldots \ge \hat{\mu}_k$ \\
\textbf{Let:} $\bar{\mu}$ be s.t. $\bar{\mu}_j = \sum_{i=1}^j \hat{\mu}_i$
\\
\textbf{Let:} $z$ be s.t. $z_j = \min\{\bar{\mu}_j - j \hat{\mu}_j,1\}$ and
$z_{k+1} = 1$\\
\textbf{If:} $\exists j$ s.t. $\frac{\bar{\mu}_j}{1 + jC} \in
[z_j,z_{j+1}]$ \+ \\
 return $a$ s.t. $\forall i,~ a_i = \max\left\{0,\mu_i - \left(-\frac{\bar{\mu}_j}{1 + jC}  +
 \bar{\mu}_j\right)/j\right\}$ \- \\
\textbf{Else:} \+ \\
 Let $j$ be the minimal index s.t. $z_j=1$ \\
set $a$ s.t. $\forall i,~~a_i = \max\{0,\mu_i
 - (-z_j + \bar{\mu}_j)/j\}$ \\
\textbf{If:} $\|a-\mu\|^2 + C  \le \|\mu\|^2$ \+ \\
 return $a$ \- \\
\textbf{Else:} \+ \\
 return $(0,\ldots,0)$ 
\end{myalgo}

The resulting pseudo-codes for Prox-SDCA is given below. We specify
the procedure while referring to $W$ as a matrix, because it is the
more natural representation. For convenience of the code, we also
maintain in $\alpha_{i,y_i}$ the value of $-\sum_{j \neq y_i}
\alpha_{i,j}$ (instead of the optimal value of $0$).

\begin{myalgo}{Prox-SDCA($(x_1,y_1)_{i=1}^n,\epsilon,\alpha,Z$)
   for solving Multiclass SVM (with smooth hinge-loss as in \eqref{eqn:SmoothMaxOfHingeConj})} 
\textbf{Define:} $\tilde{\phi}_\gamma$ as in
\eqref{eqn:SmoothMaxOfHingeConj}  \\
\textbf{Goal:} Minimize \+ \\
$P(W) = \frac{1}{n} \sum_{i=1}^n
\tilde{\phi}_\gamma((W^\top x_i) - (W^\top x_i)_{y_i}) + \lambda
\left(\frac{1}{2} \textrm{vec}(W)^\top \textrm{vec}(W) -
  \textrm{vec}(W)^\top \textrm{vec}(Z)\right)$ \- \\
\textbf{Initialize} $W=Z + \frac{1}{\lambda n} \sum_{i=1}^n 
x_i \alpha_i^\top$ \\
\textbf{Iterate:} for $t=1,2,\dots$ \+ \\
 Randomly pick $i$ \\
$\hat{W} = W - \frac{1}{\lambda n} x_i \alpha^{\top}_i$ \\
$p = x_i^\top \hat{W}$, ~ $p = p - p_{y_i}$,~ $c = \mathbf{1} -
e_{y_i}$, ~$\mu = \frac{c + p}{\gamma + \|x_i\|^2/(\lambda n)} $, 
$ C = \frac{1}{1 + \gamma \lambda n / \|x_i\|^2}$ \\
$a = \textrm{OptimizeDual}(\mu,C)$ \\
$\alpha_i = -a$, $\alpha_{y_i} = \|a\|_1$ \\
$W = \hat{W} +  \frac{1}{\lambda n} x_i \alpha^{\top}_i$ \\
\textbf{Stopping condition}: \+ \\
 let $G = 0$ \\
 for $i=1,\ldots,n$ \+ \\
          $a = W^\top x_i$, $a = a-a_{y_i}$,  $c = \mathbf{1} -
          e_{y_i}$, $b = \textrm{Project}((a+c)/\gamma)$ \\
          $G = G + \frac{\gamma}{2} (\|(a+c)/\gamma\|^2 -
          \|b-(a+c)/\gamma\|^2) + c^\top \alpha^{(t)}_i +
          \frac{\gamma}{2} (\|\alpha^{(t)}_i\|^2 -
          (\alpha^{(t)}_{i,y_i})^2)$ \- \\
  Stop if $G/n + \lambda \textrm{vec}(W)^\top \textrm{vec}(W-Z) \le \epsilon$
\end{myalgo}

\section{Experiments}

\begin{figure}
\begin{center}
\begin{tabular}{ @{} L | @{} S @{} S @{} S @{} }
$\lambda$ & \scriptsize{astro-ph} & \scriptsize{cov1} &
\scriptsize{CCAT}\\ \hline

\input{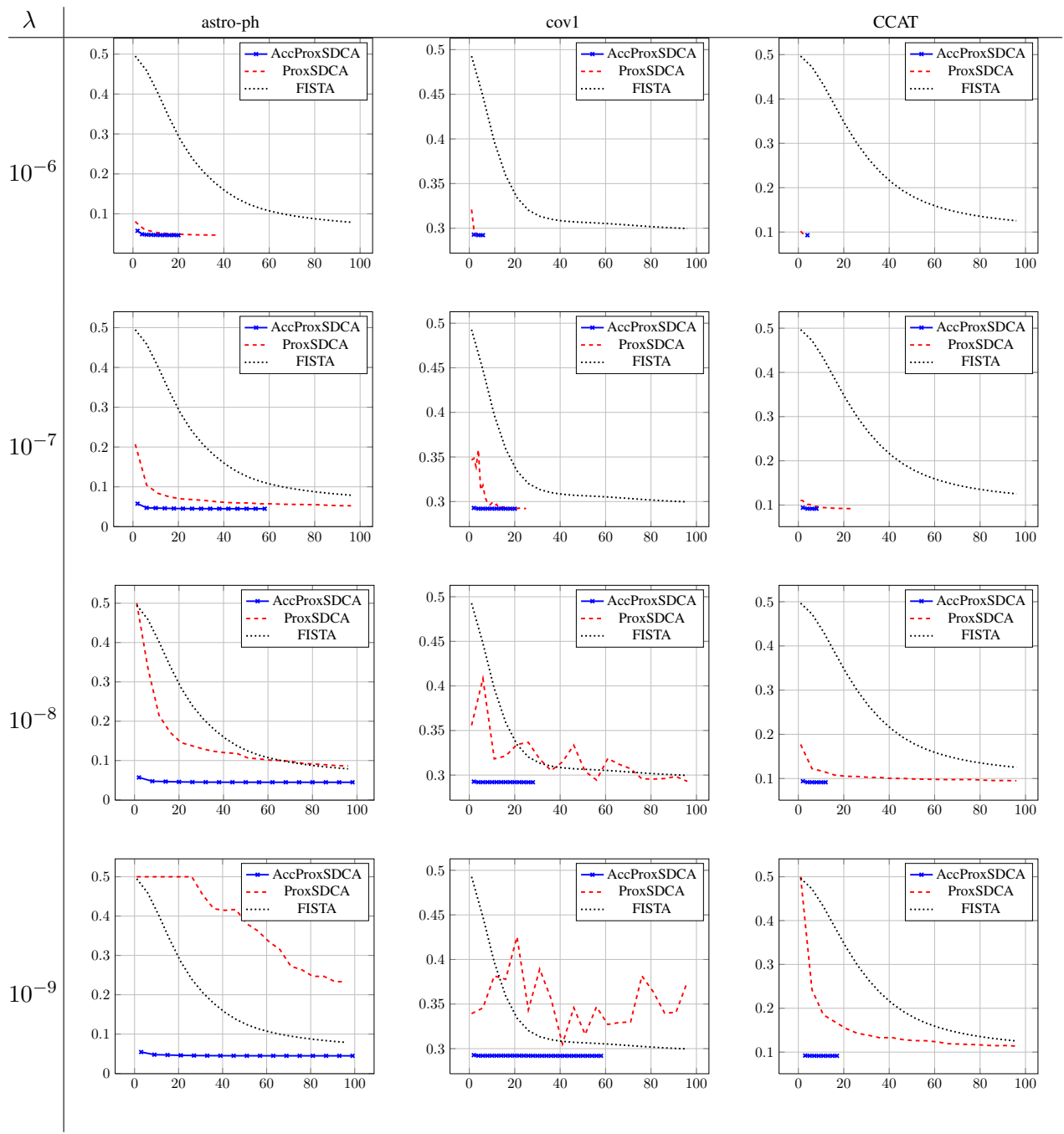}

\end{tabular}
\end{center}
\caption{\label{fig:conv}Comparing Accelerated-Prox-SDCA, Prox-SDCA,
  and FISTA for minimizing the smoothed hinge-loss ($\gamma=1$)
  with $L_1-L_2$ regularization ($\sigma = 10^{-5}$ and $\lambda$
  varies in $\{10^{-6},\ldots,10^{-9}\}$). In each of these plots, the
  y-axis is the primal objective and the x-axis is the number of
  passes through the entire training set. The three columns
  corresponds to the three data sets. The methods are terminated
  either if stopping condition is met (with $\epsilon = 10^{-3}$) or
  after 100 passes over the data.}
\end{figure}

In this section we compare Prox-SDCA, its accelerated version
Accelerated-Prox-SDCA, and the FISTA algorithm of \cite{beck2009fast},
on $L_1-L_2$ regularized loss minimization problems.

The experiments were performed on three large datasets with very
different feature counts and sparsity, which were kindly provided by
Thorsten Joachims (the datasets were also used in \cite{ShZh12-sdca}).
The astro-ph dataset classifies abstracts of papers from the physics
ArXiv according to whether they belong in the astro-physics section;
CCAT is a classification task taken from the Reuters RCV1 collection;
and cov1 is class 1 of the covertype dataset of Blackard, Jock \&
Dean. The following table provides details of the dataset
characteristics.
\begin{center}
\begin{tabular}{|r|c|c|c|c|}
        \hline
Dataset & Training Size & Testing Size & Features & Sparsity  \\ \hline 
astro-ph & $29882$ & $32487$ & $99757$ & $0.08\%$ \\
CCAT & $781265$ & $23149$ & $47236$ & $0.16\%$ \\
cov1 & $522911$ & $58101$ & $54$ & $22.22\%$ \\
\hline
\end{tabular}
\end{center}

These are binary classification problems, with each $x_i$ being a
vector which has been normalized to be $\|x_i\|_2=1$, and $y_i$ being
a binary class label of $\pm 1$. We multiplied each $x_i$ by $y_i$ and following \cite{ShZh12-sdca}, we employed the smooth hinge loss,
$\tilde{\phi}_\gamma$, as in \eqref{eqn:smoothhinge}, with $\gamma=1$.
The optimization problem we need to solve is therefore 
\[
\min_w P(w) ~~~\textrm{where}~~~
P(w) = \frac{1}{n} \sum_{i=1}^n
\tilde{\phi}_\gamma(x_i^\top w) + \frac{\lambda}{2} \|w\|_2^2 + \sigma
\|w\|_1 ~.
\]
In the experiments, we set $\sigma=10^{-5}$ and vary $\lambda$ in the
range $\{10^{-6}, 10^{-7}, 10^{-8}, 10^{-9}\}$. 

The convergence behaviors are plotted in Figure~\ref{fig:conv}. In all
the plots we depict the primal objective as a function of the number
of passes over the data (often referred to as ``epochs''). For FISTA,
each iteration involves a single pass over the data. For Prox-SDCA,
each $n$ iterations are equivalent to a single pass over the
data. And, for Accelerated-Prox-SDCA, each $n$ inner iterations are
equivalent to a single pass over the data. For Prox-SDCA and
Accelerated-Prox-SDCA we implemented their corresponding stopping
conditions and terminate the methods once an accuracy of $10^{-3}$ was
guaranteed.

It is clear from the graphs that Accelerated-Prox-SDCA yields the best
results, and often significantly outperform the other
methods. Prox-SDCA behaves similarly when $\lambda$ is relatively
large, but it converges much slower when $\lambda$ is small. This is
consistent with our theory. Finally, the relative performance of FISTA
and Prox-SDCA depends on the ratio between $\lambda$ and $n$, but in
all cases, Accelerated-Prox-SDCA is much faster than FISTA. This is
again consistent with our theory.

\section{Discussion and Open Problems}

We have described and analyzed a proximal stochastic dual coordinate
ascent method and have shown how to accelerate the procedure. The
overall runtime of the resulting method improves state-of-the-art
results in many cases of interest. 

There are two main open problems that we leave to future research.
\begin{openProblem}
  When $\frac{1}{\lambda \gamma}$ is larger than $n$, the runtime of
  our procedure becomes $\tilde{O}\left(d\sqrt{\frac{n}{\lambda
        \gamma}}\right)$.  Is it possible to derive a method whose
  runtime is $\tilde{O}\left(d\left(n+\sqrt{\frac{1}{\lambda
          \gamma}}\right)\right)$ ?
\end{openProblem}

\begin{openProblem}
  Our Prox-SDCA procedure and its analysis works for regularizers
  which are strongly convex with respect to an arbitrary
  norm. However, our acceleration procedure is designed for
  regularizers which are strongly convex with respect to the Euclidean
  norm.  Is is possible to extend the acceleration procedure to more
  general regularizers?
\end{openProblem}

\section*{Acknowledgements}
The authors would like to thank Fen Xia for careful proof-reading of the paper which helped us to correct numerous typos.
Shai Shalev-Shwartz is supported by the following grants: Intel
Collaborative Research Institute for Computational Intelligence
(ICRI-CI) and ISF 598-10. Tong Zhang is supported by the following
grants: NSF IIS-1016061, NSF DMS-1007527, and NSF IIS-1250985.

\appendix

\section{Proofs of Iteration Bounds for Prox-SDCA}

The proof technique follows that of \citet{ShalevZh2013}, but with the
required generality for handling general strongly convex regularizers
and smoothness/Lipschitzness with respect to general norms.

We prove the theorems for running Prox-SDCA while choosing $\Delta
\alpha_i$ as in Option I. A careful examination of the proof easily
reveals that the results hold for the other options as well. More
specifically, Lemma~\ref{lem:key} only requires choosing $\Delta
\alpha_i = s (u_i^{(t-1)}-\alpha_i^{(t-1)})$ as in \eqref{eqn:PC1},
and Option III chooses $s$ to optimize the bound on the right hand
side of \eqref{eqn:PC3}, and hence ensures that the choice can do no
worse than the result of Lemma~\ref{lem:key} with any $s$. The
simplification in Option IV and V employs the specific simplification
of the bound in Lemma~\ref{lem:key} in the proof of the theorems.


The key lemma is the following:
\begin{lemma} \label{lem:key} Assume that $\phi^*_i$ is
  $\gamma$-strongly-convex. For any iteration $t$, let $\E_{t}$ denote
  the expectation with respect to the randomness in choosing $i$ at
  round $t$, conditional on the value of $\alpha^{(t-1)}$. 
Then, for any iteration $t$ and any $s \in
  [0,1]$ we have
\[
\E_t[D(\alpha^{(t)})-D(\alpha^{(t-1)})] \ge  \frac{s}{n}\,
 [P(w^{(t-1)})-D(\alpha^{(t-1)})] - \left(\frac{s}{n}\right)^2
\frac{G^{(t)}}{2\lambda} ~,
\]
where
\[
G^{(t)} = \frac{1}{n} \sum_{i=1}^n \left(\|X_i\|_{D\to D'}^2 -
      \frac{\gamma(1-s)\lambda n}{s}\right) \; \E_t \left[\|u^{(t-1)}_i-\alpha^{(t-1)}_i\|_D^2\right] ,
\]
and $-u^{(t-1)}_i = \nabla \phi_i(X_i^\top w^{(t-1)})$. 
\end{lemma}
\begin{proof}
Since only the $i$'th element of $\alpha$ is updated, the improvement in the dual objective can be written as
\begin{align*}
& n[D(\alpha^{(t)}) - D(\alpha^{(t-1)})] \\
= &
\left(-\phi^*(-\alpha^{(t)}_i) - \lambda n g^*\left(v^{(t-1)} + (\lambda
  n)^{-1} X_i \Delta \alpha_i\right) \right) -
\left(-\phi^*(-\alpha^{(t-1)}_i) - \lambda n g^*\left(v^{(t-1)}\right)
\right) 
\end{align*}
The smoothness of $g^*$ implies that 
$g^*(v+ \Delta v) \leq h(v;\Delta v)$, where 
$h(v;\Delta v)
 := g^*(v) + \nabla g^*(v)^\top \Delta v +
\frac{1}{2} \|\Delta v\|_{D'}^2$. Therefore, 
\begin{align*}
& n[D(\alpha^{(t)}) - D(\alpha^{(t-1)})] \\
\geq &
\underbrace{\left(-\phi^*(-\alpha^{(t)}_i) - \lambda n h\left(v^{(t-1)}; (\lambda
  n)^{-1} X_i \Delta \alpha_i\right)\right) }_A -
\underbrace{\left(-\phi^*(-\alpha^{(t-1)}_i) - \lambda n g^*\left(v^{(t-1)}\right) \right)}_B .
\end{align*}

By the definition of the update we have for all $s \in [0,1]$ that
\begin{align} \nonumber
A &=  \max_{\Delta \alpha_i} -\phi^*(-(\alpha^{(t-1)}_i + \Delta\alpha_i)) - 
\lambda n h\left(v^{(t-1)}; (\lambda
  n)^{-1} X_i \Delta \alpha_i\right) \\
&\ge -\phi^*(-(\alpha^{(t-1)}_i + s(u^{(t-1)}_i - \alpha^{(t-1)}_i) ))
- \lambda n
h(v^{(t-1)}; (\lambda n)^{-1} s X_i (u^{(t-1)}_i -\alpha^{(t-1)}_i)) .
\label{eqn:PC1}
\end{align}

From now on, we omit the superscripts and subscripts. 
Since $\phi^*$ is $\gamma$-strongly convex, we have that
\begin{equation} \label{eqn:PC2}
\phi^*(-(\alpha+ s(u - \alpha) )) = \phi^*(s (-u) + (1-s) (-\alpha))
\le s \phi^*(-u) + (1-s) \phi^*(-\alpha) - \frac{\gamma}{2} s (1-s) \|u-\alpha\|_D^2
\end{equation}
Combining this with \eqref{eqn:PC1} and rearranging terms we obtain that
\begin{align*} 
A &\ge -s \phi^*(-u) - (1-s) \phi^*(-\alpha) + \frac{\gamma}{2} s (1-s)
\|u-\alpha\|_D^2
- \lambda n
  h(v; (\lambda n)^{-1} s X(u - \alpha) )  \\
&= -s \phi^*(-u) - (1-s) \phi^*(-\alpha) + \frac{\gamma}{2} s (1-s)
\|u-\alpha\|_{D}^2
- \lambda n g^*(v) - s w^\top X (u-\alpha) - 
\frac{s^2 \|X(u-\alpha)\|_{D'}^2}{2\lambda n} \\
&\ge -s(\phi^*(-u)+w^\top X u) + (-\phi^*(-\alpha) - \lambda
  n g^*(v)) \\
&~~~~~~~~~~+ \frac{s}{2}\left(\gamma(1-s)-\frac{s
  \|X \|_{D\to D'}^2}{\lambda n}\right)\|u-\alpha\|_D^2 + s(\phi^*(-\alpha)+
w^\top X \alpha) .
\end{align*}
Since $-u = \nabla \phi(X^\top w)$ we have 
$\phi^*(-u) +  w^\top X u = - \phi(X^\top w)$, which yields
\begin{equation} \label{eqn:PC3}
A-B \ge s\left[\phi(X^\top w) + \phi^*(-\alpha) + w^\top X \alpha +
\left(\frac{\gamma(1-s)}{2} - \frac{s
  \|X\|_{D\to D'}^2}{2\lambda n}\right) \|u-\alpha\|_D^2 \right] ~.
\end{equation}
Next note that with $w=\nabla g^*(v)$, we have $g(w)+g^*(v)= w^\top v$. Therefore:
\begin{align*} 
P(w)-D(\alpha) &= \frac{1}{n} \sum_{i=1}^n \phi_i(X_i^\top w) +
  \lambda g(w) - \left(-\frac{1}{n} \sum_{i=1}^n
  \phi^*_i(-\alpha_i) - \lambda g^*(v) \right) \\
&= \frac{1}{n} \sum_{i=1}^n \phi_i(X_i^\top w) 
+ \frac{1}{n} \sum_{i=1}^n \phi^*_i(-\alpha_i)  + \lambda w^\top v \\
&= \frac{1}{n} \sum_{i=1}^n \left( \phi_i(X_i^\top w) +
  \phi^*_i(-\alpha_i) 
+ w^\top X_i \alpha_i \right)  .
\end{align*}
Therefore, if we take expectation of \eqref{eqn:PC3} w.r.t. the choice
of $i$ we obtain that
\[
\frac{1}{s}\, \E_t[A-B] \ge  [P(w)-D(\alpha)] - \frac{s}{2\lambda
    n} \cdot \underbrace{\frac{1}{n} \sum_{i=1}^n \left(\|X_i\|_{D\to D'}^2 -
      \frac{\gamma(1-s)\lambda n}{s}\right) \E_t[\|u_i-\alpha_i\|_D^2] }_{= G^{(t)}} .
\]
We have obtained that
\begin{equation} \label{eqn:DualSObyGap}
\frac{n}{s}\, \E_t[D(\alpha^{(t)})-D(\alpha^{(t-1)})] \ge
[P(w^{(t-1)})-D(\alpha^{(t-1)})] - \frac{s\,G^{(t)}}{2\lambda n} ~.
\end{equation}
Multiplying both sides by $s/n$ concludes the proof of the lemma.
\end{proof}


Equipped with the above lemmas we are ready to prove
\thmref{thm:smooth} and \thmref{thm:HighProbsmooth}. 

\begin{proof}[Proof of \thmref{thm:smooth}]
The assumption that $\phi_i$ is $(1/\gamma)$-smooth implies that
$\phi_i^*$ is $\gamma$-strongly-convex. 
We will apply \lemref{lem:key} with 
\[
s =  \frac{n}{n + R^2/(\lambda \gamma)} =
\frac{\lambda n \gamma}{R^2 + \lambda n \gamma } \in [0,1] ~.
\] Recall that
$\|X_i\|_{D\to D'} \le R$. Therefore, 
the choice of $s$ implies that 
\[
\|X_i\|_{D\to D'}^2 -
      \frac{\gamma(1-s)\lambda n}{s} \le R^2 - \frac{1-s}{s/(\lambda
        n \gamma )} = R^2 - R^2 = 0 ~,
\] and hence $G^{(t)} \le 0$ for
all $t$. This yields, 
\begin{equation} \label{eqn:lem1CorForSmooth}
\E_t[D(\alpha^{(t)})-D(\alpha^{(t-1)})] \ge  \frac{s}{n}\,
(P(w^{(t-1)})-D(\alpha^{(t-1)})) ~.
\end{equation}
Taking expectation
of both sides with respect to the randomness at previous rounds, and
using the law of total expectation, we obtain that
\begin{equation} \label{eqn:Ialsoneedthis}
\E[D(\alpha^{(t)})-D(\alpha^{(t-1)})] \ge  \frac{s}{n}\,
\E[P(w^{(t-1)})-D(\alpha^{(t-1)})] ~.
\end{equation}
But since $\epsilon_D^{(t-1)} := D(\alpha^*)-D(\alpha^{(t-1)}) \le P(w^{(t-1)})-D(\alpha^{(t-1)})$ and $D(\alpha^{(t)})-D(\alpha^{(t-1)})
= \epsilon_D^{(t-1)} - \epsilon_D^{(t)}$, we obtain that 
\[
\E[ \epsilon_D^{(t)} ] \le \left(1 -
  \tfrac{s}{n}\right)\E[\epsilon_D^{(t-1)}] \le \left(1 -
  \tfrac{s}{n}\right)^t \,\epsilon_D^{(0)} \le
\epsilon_D^{(0)}\,e^{-\frac{st}{n}}~.
\]
Therefore, whenever
\[
t \ge \frac{n}{s}\,\log(\epsilon_D^{(0)}/\epsilon_D) = \left(n +
  \tfrac{R^2}{\lambda \gamma}\right) \, \log(\epsilon_D^{(0)}/\epsilon_D) ~,
\]
we are guaranteed that $\E[ \epsilon_D^{(t)} ]$ would be smaller than
$\epsilon_D$.  

Using again \eqref{eqn:Ialsoneedthis}, we can also
obtain that
\begin{equation}
\E[P(w^{(t)})-D(\alpha^{(t)})]  \le \frac{n}{s}
\E[D(\alpha^{(t+1)})-D(\alpha^{(t)})] = \frac{n}{s}
\E[\epsilon_D^{(t)} - \epsilon_D^{(t+1)}] \le \frac{n}{s} \E[\epsilon_D^{(t)}] . \label{eqn:dgap-bound-smooth}
\end{equation}
So, requiring $\E[\epsilon_D^{(t)}] \le \frac{s}{n} \epsilon_P$ we obtain
an expected duality gap of at most $\epsilon_P$. This means that we should
require
\[
t \ge \left(n +
  \tfrac{R^2}{\lambda \gamma}\right) \, \log( (n + \tfrac{R^2}{\lambda \gamma})   \cdot \tfrac{\epsilon_D^{(0)}}{\epsilon_P}) ~,
\]
which proves the first part of \thmref{thm:smooth}. 

Next, we sum the first inequality of \eqref{eqn:dgap-bound-smooth} over $t=T_0+1,\ldots,T$ to obtain
\[
\E\left[ \frac{1}{T-T_0} \sum_{t=T_0+1}^{T} (P(w^{(t)})-D(\alpha^{(t)}))\right] \le 
\frac{n}{s(T-T_0)} \E[D(\alpha^{(T+1)})-D(\alpha^{(T_0+1)})] .
\]
Now, if we choose $\bar{w},\bar{\alpha}$ to be either the average
vectors or a randomly chosen vector over $t \in \{T_0+1,\ldots,T\}$,
then the above implies
\begin{align*}
\E[ P(\bar{w})-D(\bar{\alpha})] &\le 
\frac{n}{s(T-T_0)} \E[D(\alpha^{(T+1)})-D(\alpha^{(T_0+1)})]  \\
&\le \frac{n}{s(T-T_0)} \E[\epsilon_D^{(T_0+1)})] \\
&\le \frac{n}{s(T-T_0)}
\epsilon_D^{(0)} e^{-\frac{sT_0}{n}}. 
\end{align*}
It follows that in order to obtain a result of 
$\E[ P(\bar{w})-D(\bar{\alpha})] \le \epsilon_P$, we need to have
\[
T_0 \ge \frac{n}{s} \log\left( \frac{n \epsilon_D^{(0)}}{s (T-T_0) \epsilon_P} \right) ~.
\]
In particular, the choice of $T-T_0 = \frac{n}{s}$ and $T_0 = \frac{n}{s}
\log(\epsilon_D^{(0)}/\epsilon_P)$ satisfies the above requirement. 
\end{proof}

\begin{proof}[Proof of \thmref{thm:HighProbsmooth}]
Define $t_0 = \lceil \frac{n}{s}
  \log(2\epsilon_D^{(0)}/\epsilon_D) \rceil$. 
  The proof of \thmref{thm:smooth}
  implies that for every $t$, $\E[\epsilon_D^{(t)}] \le
  \epsilon_D^{(0)}\,e^{-\frac{st}{n}}$. By Markov's inequality, with
  probability of at least $1/2$ we have $\epsilon_D^{(t)} \le
  2\epsilon_D^{(0)}\,e^{-\frac{st}{n}}$. Applying it for $t=t_0$ we get that
  $\epsilon_D^{(t_0)} \le \epsilon_D$ with probability of at least
  $1/2$. Now, lets apply the same argument again, this time with the
  initial dual sub-optimality being $\epsilon_D^{(t_0)}$. Since the dual
  is monotonically non-increasing, we have that $\epsilon_D^{(t_0)}
  \le \epsilon_D^{(0)}$. Therefore, the same argument tells us that
  with probability of at least $1/2$ we would have that
  $\epsilon_D^{(2t_0)} \le \epsilon_D$. Repeating this 
  $\lceil \log_2(1/\delta) \rceil$ times, we obtain that with
  probability of at least $1-\delta$, for some $k$ we have that
  $\epsilon_D^{(kt_0)} \le \epsilon_D$. Since the dual is monotonically
  non-decreasing, the claim about the dual sub-optimality follows.

Next, for the duality gap, using \eqref{eqn:lem1CorForSmooth} we have
that for every $t$ such that $\epsilon_D^{(t-1)} \le \epsilon_D$ we have
\[
P(w^{(t-1)})-D(\alpha^{(t-1)}) ~\le~
\frac{n}{s} \, \E[D(\alpha^{(t)})-D(\alpha^{(t-1)})] \le
\frac{n}{s} \, \epsilon_D ~.
\]
This proves the second claim of \thmref{thm:HighProbsmooth}. 

For the last claim, suppose that at round $T_0$ we have
$\epsilon_D^{(T_0)} \le \epsilon_D$. Let $T = T_0 + n/s$.  It follows
that if we choose $t$ uniformly at random from $\{T_0,\ldots,T-1\}$,
then $\E[ P(w^{(t)})-D(\alpha^{(t)})] \le \epsilon_D$. By Markov's
inequality, with probability of at least $1/2$ we have $
P(w^{(t)})-D(\alpha^{(t)}) \le 2\epsilon_D$. Therefore, if we choose
$\log_2(2/\delta)$ such random $t$, with probability $\ge 1-\delta/2$,
at least one of them will have $ P(w^{(t)})-D(\alpha^{(t)}) \le
2\epsilon_D$. Combining with the first claim of the theorem, choosing
$\epsilon_D = \epsilon_P/2$, and applying the union bound, we conclude
the proof of the last claim of \thmref{thm:HighProbsmooth}.
\end{proof}

\bibliographystyle{plainnat}
\bibliography{curRefs}

\end{document}